\pdfoutput=1
\documentclass[nohyperref]{article}
\usepackage[dvipsnames]{xcolor}
\usepackage{xspace, amsmath, amssymb, centernot}
\usepackage{graphicx, placeins, subcaption}
\usepackage{booktabs, multirow}
\usepackage[pagebackref,breaklinks,colorlinks]{hyperref}
\usepackage{amsthm, siunitx}
\usepackage{thmtools}
\usepackage{thm-restate}
\usepackage[accepted]{icml2023}

\newif\ifcomments
\commentstrue
\ifcomments
    \newcommand\todo[1]{\textcolor{red}{[TODO: #1]}}
    \newcommand\outline[1]{{\leavevmode\color{darkgray}#1}}
    
    \newcommand{\irena}[1]{\textcolor{blue}{[IG: #1]}}
    \newcommand{\shiori}[1]{\textcolor{cyan}{[SS: #1]}}
    \newcommand{\pw}[1]{\textcolor{ForestGreen}{[PW: #1]}}
    \newcommand{\pl}[1]{\textcolor{red}{[PL: #1]}}
\else
    \newcommand\todo[1]{}
    \newcommand\outline[1]{}
    
    \newcommand{\irena}[1]{}
    \newcommand{\shiori}[1]{}
    \newcommand{\pw}[1]{}
    \newcommand{\pl}[1]{} 
\fi

\newcommand{\tightparagraph}[1]{\vspace{0.3em}\noindent\textbf{#1}\hspace{0.25em}}
\newenvironment{proofsketch}{\begin{proof}[Proof sketch.]}{\end{proof}}

\newcommand{\ie}{{i.e.,}\xspace}
\newcommand{\eg}{{e.g.,}\xspace}
\newcommand{\iid}{\emph{i.i.d.}\xspace}
\newcommand{\ood}{out-of-domain\xspace}
\newcommand{\oodnodash}{out of domain\xspace}
\newcommand{\Ood}{Out-of-domain\xspace}
\newcommand{\OoD}{Out-of-Domain\xspace}
\newcommand{\id}{in-domain\xspace}
\newcommand{\Id}{In-domain\xspace}

\newcommand{\domainindependent}{domain-independent\xspace}
\newcommand{\domaindependent}{domain-dependent\xspace}
\newcommand{\domaindependentnodash}{domain dependent\xspace}

\newcommand{\notcrossdomain}{generic\xspace}

\newcommand{\shelf}{generic\xspace} 
\newcommand{\Shelf}{Generic\xspace}
\newcommand{\varratioword}{variance ratio\xspace}

\newcommand{\iwc}{{iWildCam}\xspace}
\newcommand{\cam}{{Camelyon17}\xspace}
\newcommand{\birds}{{BirdCalls}\xspace}
\newcommand{\iwildcam}{{\sc iWildCam2020-WILDS}\xspace}
\newcommand{\camelyon}{{\sc Camelyon17-WILDS}\xspace}
\newcommand{\birdcalls}{{\sc BirdCalls}\xspace}

\newcommand{\cp}{{Copy-Paste}\xspace}
\newcommand{\cpsamey}{{\cp \small{(Same Y)}}\xspace}
\newcommand{\cpall}{{\cp \small{(All Backgrounds)}}\xspace}
\newcommand{\cpallr}{{\cp + Jitter \small{(All Regions)}}\xspace}
\newcommand{\cpsamer}{{\cp + Jitter \small{(Region)}}\xspace}
\newcommand{\stainj}{{Stain Color Jitter}\xspace}

\newcommand{\indep}{\perp \!\!\! \perp}
\newcommand{\nindep}{\centernot{\indep}}
\newcommand\ix{{(i)}}

\newcommand{\defeq}{\triangleq}

\newcommand\dom{d}
\newcommand\Dom{\sD}
\newcommand\ndom{D}
\newcommand\Domall{\Dom^\mathsf{all}}
\newcommand\Domtrain{\Dom^\mathsf{train}}
\newcommand\Domtest{\Dom^\mathsf{test}}
\newcommand\Domval{\Dom^\mathsf{val}}
\newcommand\Pmeta{P}
\newcommand\Pmetadom{P(\dom)}
\newcommand\Ptrain{P^\mathsf{train}}
\newcommand\Ptrainhat{\hat{P}^\mathsf{train}}

\newcommand\Dtrain{\mathcal{D}^\mathsf{train}}

\newcommand\domsupscript{\scalebox{0.7}{(\dom)}}

\newcommand\xdom{x_\mathsf{dom}}
\newcommand\xobj{x_\mathsf{obj}}
\newcommand\xcore{x_\mathsf{d:robust}}
\newcommand\xspu{x_\mathsf{d:spu}}
\newcommand\xnoise{x_\mathsf{noise}}
\newcommand\comb[1]{f(#1)}

\newcommand\augtgt{A_\mathsf{tgt}}
\newcommand\auginv{A_\mathsf{inv}}
\newcommand\augstd{A_\mathsf{gen}}
\newcommand\augx[1]{#1^\prime}

\newcommand\Paughat{\hat{P}^\mathsf{train}_\mathsf{aug}}
\newcommand\Ptgt{P^\mathsf{train}_\mathsf{tgt}}
\newcommand\Pinv{P^\mathsf{train}_\mathsf{inv}}
\newcommand\Pstd{P^\mathsf{train}_\mathsf{gen}}

\newcommand\mud{\mu^{\domsupscript}}
\newcommand\mudT{\mu^{\domsupscript \top}}
\newcommand\mucore{\mu_\mathsf{robust}^{\domsupscript}}
\newcommand\muspu{\mu_\mathsf{spu}^{\domsupscript}}
\newcommand\mucoreT{\mu_\mathsf{robust}^{\domsupscript \top}}

\newcommand\sigmacore{\sigma}
\newcommand\sigmaspu{\sigma}
\newcommand\sigmay{\sigmaeps}

\newcommand\sigmaeps{\sigma_\varepsilon}
\newcommand\taucore{\tau}
\newcommand\tauspu{\tau}

\newcommand\dimcore{p_\mathsf{robust}}
\newcommand\dimspu{p_\mathsf{spu}}
\newcommand\dimnoise{p_\mathsf{noise}}
\newcommand\dimdom{p_\mathsf{dom}}

\newcommand\thetastar{\beta}

\newcommand\thetastarobj{\beta_\mathsf{obj}}
\newcommand\thetastarnoise{\beta_\mathsf{noise}}
\newcommand\thetastarcore{\beta_\mathsf{robust}}
\newcommand\thetastarspu{\beta_\mathsf{spu}}
\newcommand\thetahat{\hat{\theta}}

\newcommand\thetahataug{\hat{\theta}^\mathsf{(aug)}}
\newcommand\thetahatnoaug{\hat{\theta}^\mathsf{(unaug)}}
\newcommand\thetahatstd{\hat{\theta}^\mathsf{(gen)}}
\newcommand\thetahatinv{\hat{\theta}^\mathsf{(inv)}}
\newcommand\thetahattgt{\hat{\theta}^\mathsf{(tgt)}}

\newcommand\thetadom{\theta_\mathsf{dom}}
\newcommand\thetacore{\theta_\mathsf{robust}}
\newcommand\thetaspu{\theta_\mathsf{spu}}
\newcommand\thetastardom{\beta_\mathsf{dom}}

\newcommand\thetahatnoaugdom{\hat{\theta}^\mathsf{(unaug)}_\mathsf{dom}}

\newcommand\thetahattgtcore{\hat{\theta}^\mathsf{(tgt)}_\mathsf{robust}}
\newcommand\thetahattgtspu{\hat{\theta}^\mathsf{(tgt)}_\mathsf{spu}}

\newcommand\thetabest{{\theta^*}}

\newcommand\thetastarT{\beta^{\top}}
\newcommand\thetastarcoreT{\beta^{\top}_\mathsf{robust}}

\newcommand\thetahatnoaugT{\hat{\theta}^{\mathsf{(unaug)}\top}}

\newcommand\thetahatinvT{\hat{\theta}^{\mathsf{(inv)}\top}}
\newcommand\thetahattgtT{\hat{\theta}^{\mathsf{(tgt)}\top}}
\newcommand\thetahattgtcoreT{\hat{\theta}^{\mathsf{(tgt)}T_\mathsf{robust}}}

\newcommand\Sigmacore{\Sigma_\mathsf{robust}}
\newcommand\Sigmatgt{\Sigma^\mathsf{(tgt)}}
\newcommand\Tcore{T_\mathsf{robust}}
\newcommand\Mtgt{M^\mathsf{(tgt)}}
\newcommand\Mcore{M_\mathsf{robust}}

\newcommand\lambdamin{\lambda_\mathsf{min}}
\newcommand\lambdamax{\lambda_\mathsf{max}}

\newcommand\Rood{R^\mathsf{OOD}}
\newcommand\Rid{R^\mathsf{ID}}
\newcommand\varratio{\gamma^2}

\newcommand\midd{\;\middle|\;}
\newcommand\zd{{z^{(d)}}}

\newcommand\sD{\ensuremath{\mathcal{D}}}

\newcommand\BR{\ensuremath{\mathbb{R}}}

\DeclareMathOperator*{\diag}{diag} 

\DeclareMathOperator*{\argmin}{arg\,min}

\newcommand\refeqn[1]{(\ref{eqn:#1})}

\newcommand\refsec[1]{Section~\ref{sec:#1}}

\newcommand\reffig[1]{Figure~\ref{fig:#1}}

\newcommand\reftab[1]{Table~\ref{tab:#1}}

\newcommand\refapp[1]{Appendix~\ref{sec:#1}}
\newcommand\refthm[1]{Theorem~\ref{thm:#1}}
\newcommand\refthms[2]{Theorems~\ref{thm:#1} and~\ref{thm:#2}}
\newcommand\reflem[1]{Lemma~\ref{lem:#1}}

\newcommand\refprop[1]{Proposition~\ref{prop:#1}}

\newcommand\refcor[1]{Corollary~\ref{cor:#1}}

\ifthenelse{\isundefined{\definition}}{\newtheorem{definition}{Definition}}{}
\ifthenelse{\isundefined{\assumption}}{\newtheorem{assumption}{Assumption}}{}
\ifthenelse{\isundefined{\hypothesis}}{\newtheorem{hypothesis}{Hypothesis}}{}
\ifthenelse{\isundefined{\proposition}}{\newtheorem{proposition}{Proposition}}{}
\ifthenelse{\isundefined{\theorem}}{\newtheorem{theorem}{Theorem}}{}
\ifthenelse{\isundefined{\lemma}}{\newtheorem{lemma}{Lemma}}{}
\ifthenelse{\isundefined{\corollary}}{\newtheorem{corollary}{Corollary}}{}
\ifthenelse{\isundefined{\alg}}{\newtheorem{alg}{Algorithm}}{}
\ifthenelse{\isundefined{\example}}{\newtheorem{example}{Example}}{}
\newcommand{\E}{\ensuremath{\mathbb{E}}} 

\icmltitlerunning{\OoD Robustness via Targeted Augmentations}

\begin{document}

\twocolumn[
\icmltitle{\OoD Robustness via Targeted Augmentations}

\icmlsetsymbol{equal}{*}

\begin{icmlauthorlist}
\icmlauthor{Irena Gao}{equal,st}
\icmlauthor{Shiori Sagawa}{equal,st}
\icmlauthor{Pang Wei Koh}{uw,goog}
\icmlauthor{Tatsunori Hashimoto}{st}
\icmlauthor{Percy Liang}{st}
\end{icmlauthorlist}

\icmlaffiliation{st}{Stanford University}
\icmlaffiliation{goog}{Google Brain}
\icmlaffiliation{uw}{University of Washington}

\icmlcorrespondingauthor{Irena Gao}{irena@cs.stanford.edu}
\icmlcorrespondingauthor{Shiori Sagawa}{ssagawa@cs.stanford.edu}

\icmlkeywords{Machine Learning, ICML}

\vskip 0.3in
]

\printAffiliationsAndNotice{\icmlEqualContribution} 

\begin{abstract}
   Models trained on one set of domains often suffer performance drops on unseen domains, \eg when wildlife monitoring models are deployed in new camera locations.
In this work, we study principles for designing data augmentations for \ood (OOD) generalization. 
In particular, we focus on real-world scenarios in which some \emph{\domaindependent} features are \emph{robust}, \ie some features that vary across domains are predictive OOD.
For example, in the wildlife monitoring application above, image backgrounds vary across camera locations but indicate habitat type, which helps predict the species of photographed animals. 
Motivated by theoretical analysis on a linear setting, we propose \emph{targeted augmentations}, which selectively randomize spurious \domaindependent features while preserving robust ones.
We prove that targeted augmentations improve OOD performance, allowing models to generalize better with fewer domains.
In contrast, existing approaches such as \shelf augmentations, which fail to randomize \domaindependent features, and domain-invariant augmentations, which randomize all \domaindependent features, both perform poorly OOD.
In experiments on three real-world datasets, targeted augmentations set new state-of-the-art OOD performances by
3.2--15.2 percentage points.
\end{abstract}


\section{Introduction}\label{sec:intro}
Real-world machine learning systems are often deployed on domains unseen during training. 
However, distribution shifts between domains can substantially degrade model performance.
For example, in wildlife conservation, where ecologists use machine learning to identify animals photographed by static camera traps,
models suffer large performance drops on cameras not included during training~\citep{beery2018recognition}.
\Ood (OOD) generalization in such settings remains an open challenge, with recent work showing that current methods do not perform well \citep{gulrajani2020search,koh2021wilds}.

One approach to improving robustness is data augmentation,
but how to design augmentations for OOD robustness remains an open question.
Training with \emph{\shelf augmentations} developed for \id (ID) performance (\eg random crops and rotations) has sometimes improved OOD performance, 
but gains are often small and inconsistent across datasets~\citep{gulrajani2020search,wiles2021fine,hendrycks2021many}.
Other work has designed augmentations to encourage \emph{domain invariance}, but gains can be limited, especially on real-world shifts \citep{yan2020improve,zhou2020deep,gulrajani2020search,ilse2021selecting,yao2022improving}.
Some applied works have shown that heuristic, application-specific augmentations can improve OOD performance on specific tasks \citep{tellez2018whole,tellez2019quantifying,ruifrok2001quantification}.
However, it is unclear what makes these augmentations successful or how to generalize the approach to other OOD problems.

In this work, we study principles for designing data augmentations for OOD robustness. 
We focus on real-world scenarios in which there are some \emph{\domaindependent} features that are \emph{robust}, \ie where some features that vary across domains are predictive \ood.
For example, in the wildlife monitoring application above, image backgrounds vary across cameras but also contain features that divulge the static camera's habitat (\eg savanna, forest, etc.).
This information is predictive across all domains, as wild animals only live in certain habitats; it can also be necessary for prediction when foreground features are insufficient (\eg when animals are blurred or obscured).
These real-world scenarios represent a shift from prior work, which typically assumes that only \emph{\domainindependent} features that are stable across domains, like the animal foregrounds, are necessary for prediction.

How might data augmentations improve OOD robustness in such settings? 
We first theoretically analyze a linear regression setting and show that unaugmented models incur high OOD risk when the OOD generalization problem is underspecified, \ie when there are fewer training domains than the dimensionality of the \domaindependent features.
This insight motivates \emph{targeted augmentations}, which selectively randomize spurious \domaindependent features while preserving robust ones, reducing the effective dimensionality and bringing the problem to a fully specified regime.
In this linear regression setting, we prove that targeted augmentations improve OOD risk in expectation, allowing us to generalize with fewer domains.
In contrast, existing approaches such as \shelf augmentations, which fail to randomize \domaindependent features, and domain-invariant augmentations, which randomize all \domaindependent features, both suffer high OOD risk:
the former fails to address the underspecification issue, and the latter eliminates robust \domaindependent features that are crucial for prediction.
To our knowledge, our analysis is the first to characterize how different augmentation strategies affect OOD risk and its scaling with the number of domains.
It also introduces a natural theoretical setting for OOD generalization. 
Prior work studies worst-case shifts induced by adversarially selected training domains \citep{rosenfeld2020risks,chen2021iterative}.
Here, domains are not adversarial; training domains are sampled from the same domain distribution as test domains.
However, finite samples of training domains still induce challenging shifts between the training and test data.

Empirically, we show targeted augmentations are effective on three real-world datasets spanning biomedical and wildlife monitoring applications: \camelyon \citep{bandi2018detection,koh2021wilds}, \iwildcam \citep{beery2021iwildcam,koh2021wilds},
and \birdcalls, which we curate from ornithology datasets \citep{amanda_navine_2022_7078499,w_alexander_hopping_2022_7079124,stefan_kahl_2022_7079380}.
Targeted augmentations outperform both \shelf augmentations and domain invariance baselines to achieve state-of-the-art by substantial margins:
33.3\% $\to$ 36.5\% on \iwildcam, 75.3\% $\to$ 90.5\% on \camelyon, and 31.8\% $\to$ 37.8\% on \birdcalls.
On \iwildcam, targeted augmentations also confer \emph{effective robustness} \citep{miller2021accuracy}.
Overall, our work derives principles for designing data augmentations that can substantially improve \ood performance.

\begin{figure*}[t]
  \centering
  \includegraphics[width=\textwidth]{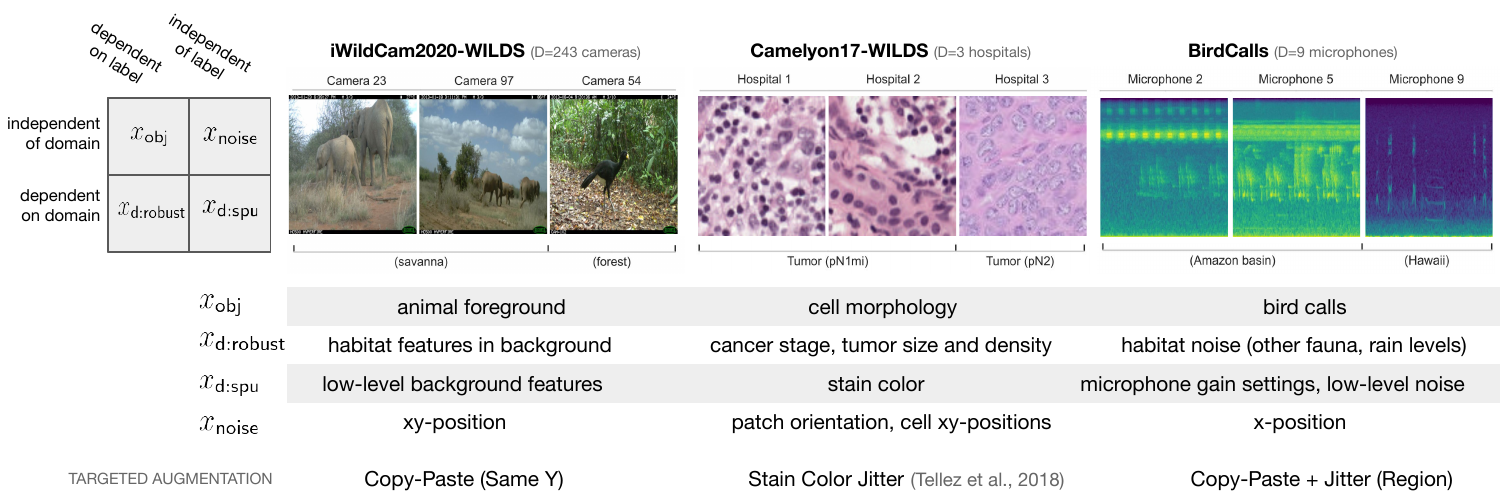}
  \caption{We model inputs as $x = \comb{\xobj, \xcore, \xspu, \xnoise}$, where each of the four types of features are either (i) dependent on the domain $\dom$ or not and (ii) dependent on the output label $y$ or not, both in the population $\Pmeta$.
  We study targeted augmentations, which randomize $\xspu$ but preserve $\xcore$, and 
  we consider three real-world datasets \citep{beery2021iwildcam,bandi2018detection,koh2021wilds}, each of which have both robust and spurious \domaindependent features.}
  \label{fig:datasets}
  \vspace{-1em}
\end{figure*}

\section{Problem setting}\label{sec:setup}
\tightparagraph{Domain generalization.}
In domain generalization, our goal is to generalize to domains unseen during training.
In particular, we seek a model $\theta\in\Theta$ that minimizes the risk under a distribution $\Pmeta$, where
\begin{equation}
  \Rood(\theta)\triangleq\E_\Pmeta[\ell(\theta; (x,y))],
\label{eqn:rood}
\end{equation}
and $\Pmeta$ comprises data from all possible domains $\Domall$:
\begin{equation}
    \Pmeta(x,y) = \sum_{\dom \in \Domall} P(x,y\mid\dom) \Pmetadom,
\label{eqn:pmeta}
\end{equation}
where we assume $\Domall$ is countable to keep notation simple.
To obtain training domains $\Domtrain\subseteq\Domall$, we sample $\ndom$ domains without replacement from $\Pmetadom$. 
This yields the training distribution comprising $\Dtrain$,
\begin{equation}
  \Ptrain(x,y) = \sum_{\dom \in \Domtrain} P(x,y\mid\dom) \Ptrain(d),
\label{eqn:ptrain}
\end{equation}
where $\Ptrain(\dom)$ is the probability of drawing domain $\dom$ from the training domains $\Domtrain$ at training time.
The challenge is to generalize from the sampled training domains $\Domtrain$ to all possible domains $\Domall$ that make up the underlying domain distribution.
In real-world experiments and simulations, we estimate OOD performance by evaluating on held-out domains $\Domtest$, where $\Domtest \cap \Domtrain = \emptyset$. 

\tightparagraph{Feature decomposition.}
In many real-world shifts, such as those in \refsec{datasets}, \domaindependent features contain predictive information that generalizes across all domains. 
To capture such settings, we introduce the feature decomposition $x = \comb{\xobj, \xnoise, \xcore, \xspu}$ for some complex function $\comb{\cdot}$ (\reffig{datasets} left).
$x$ lies in pixel space, while the features live in some abstract feature space.
We split these features along two axes: whether they are \emph{robust} (\ie predictive \ood), and whether they are \emph{\domaindependentnodash} (\ie varying across domains).
We formalize these two criteria by (in)dependence with label $y$ and domain $\dom$, respectively, in the population $\Pmeta$:
\begin{equation}
\begin{aligned}
  \xobj, \xcore \nindep y&\\
  \xnoise,\xspu \indep y&\\
  \xcore,\xspu \nindep \dom&\\
  \xobj,\xnoise \indep \dom&.
\end{aligned}
\end{equation}
For example, $y$ depends on robust features $\xobj$ and $\xcore$, but is independent of non-robust features $\xnoise$ and $\xspu$, which yields $\Pmeta(y\mid x) = \Pmeta(y \mid \xobj,\xcore)$.
Note that prior work typically only considers domain-invariant $\xobj$ relevant for $y$; however, \domaindependent $\xcore$ is also useful.

We note that the independencies above need not hold in the training distribution $\Ptrain$ due to finite-domain effects.
Recall that $\Ptrain$ is a mixture of $\ndom$ domains.
While $\xspu \indep y$ in $\Pmeta$, when $\ndom$ is small, some $\xspu$ may be correlated with $y$ in $\Ptrain$.
This leads models to learn such features and generalize poorly \ood.

\subsection{Real-world datasets}\label{sec:datasets}
We study three real-world datasets (\reffig{datasets} right), which have both robust and spurious \domaindependent features.

\tightparagraph{Species classification from camera trap images (\iwildcam).}
In \iwc~\cite{beery2021iwildcam,koh2021wilds}, the task is to classify an animal species $y$ from an image $x$ captured by a static camera trap $\dom$.
There are 243 cameras in $\Domtrain$.
Images from the same camera share nearly identical backgrounds.
While low-level details of each domain's background are generally spurious (\eg whether there are two trees or three), backgrounds also contain {habitat} features $\xcore$, which are predictive across domains.
For example, in \reffig{datasets}, cameras 23 and 97 are installed in dry Kenyan savannas, while camera 54 observes a leafy Guatemalan forest.
The two regions have different label distributions: in practice, wild African elephants are very unlikely to set foot in Guatemala.
Further, habitat features are often necessary for prediction; foregrounds are often blurry or occluded (see \reffig{iwildcam:dataset_sample}),
so randomizing all \domaindependent features discards useful information.

\tightparagraph{Tumor identification in histopathology slides (\camelyon).}
In \cam~\cite{bandi2018detection,koh2021wilds}, the task is to classify whether a patch of a histopathology slide contains a tumor.
Slides are contributed by hospitals $\dom$.
Variations in imaging technique result in domain-specific stain colorings, which spuriously correlate with $y$ in the training set (see \reffig{camelyon_histogram}).
Domains also vary in distributions of patient cancer stage.
In \cam's 3 training hospitals, most patients in Hospitals 1 and 2 have earlier-stage pN1 breast cancer, whereas nearly half of the patients in Hospital 3 have later-stage pN2 stage cancer.
The pN stage relates to the size and number of lymph node metastases, which is correlated with other histological tumor features. 
These useful tumor features thus depend on both $\dom$ and $y$.

\tightparagraph{Bird species recognition from audio recordings (\birdcalls).}
To monitor bird populations, ornithologists use machine learning to identify birds by their calls in audio recordings.
However, generalizing to recordings from new microphones can be challenging~\citep{joly2021overview}.
We introduce a new bird recognition dataset curated from publicly released data (see \refapp{app:birds} for details).
The task is to identify the bird species $y$ vocalizing in audio clip $x$ recorded by microphone $\dom$.
There are 9 microphones in $\Domtrain$, which vary in their model and location.
While low-level noise and microphone settings (\eg gain levels) only spuriously correlate with $y$, other background noises indicate habitat, like particular insect calls in the Amazon Basin that are absent from other regions (\reffig{datasets}).
As in \iwc, these habitat indicators reliably predict $y$.
We train models on mel-spectrograms of audio clips.

\begin{figure*}[bt]
    \centering
    \includegraphics[width=\textwidth]{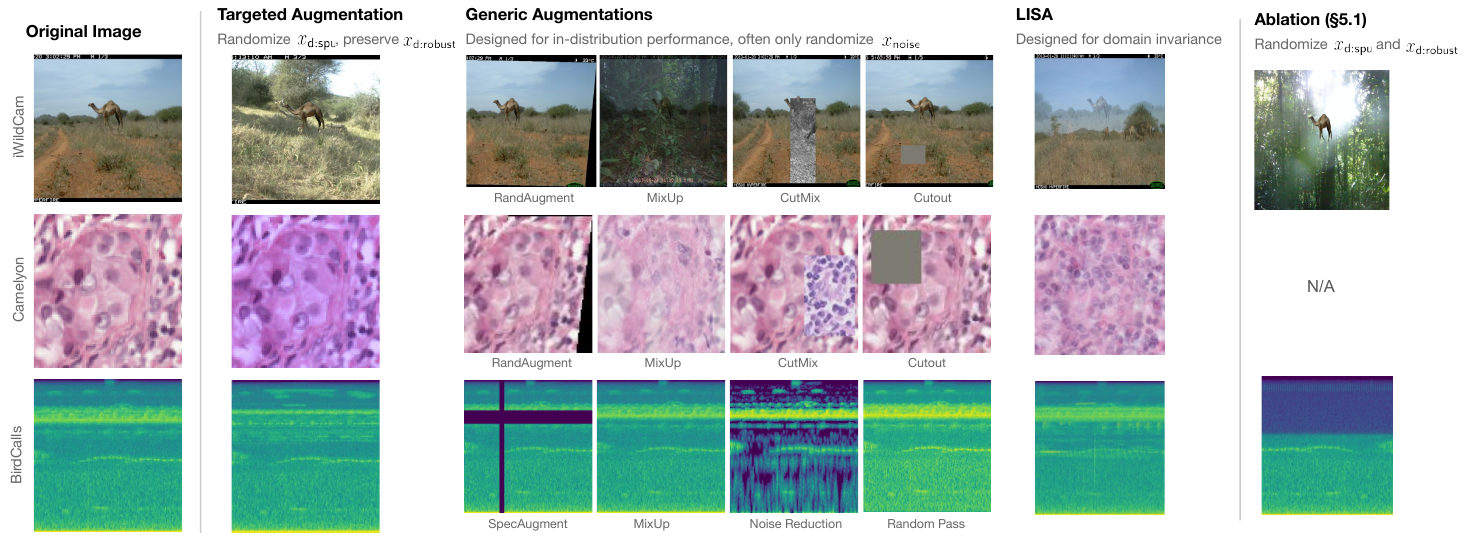}
    \caption{Augmentation examples for the three real-world datasets, including targeted augmentations \cpsamey for \iwc, \stainj for \cam, and \cpsamer for \birds. 
    Targeted augmentations randomize $\xspu$ but preserve $\xcore$. 
    In \refsec{ablations}, we compare to modified \cp augmentations in the ablation column.\label{fig:augmentations}}
    \vspace{-0.5em}
\end{figure*}
\section{Data augmentation}\label{sec:augmentations}
\tightparagraph{Augmentation types.}
We use the feature decomposition from \refsec{setup} to model three types of data augmentations.
\emph{\Shelf augmentations} designed for \id settings often do not randomize \domaindependent features.
For example, horizontal flips modify object orientation; this feature varies across examples but is typically distributed similarly across domains.
We model \shelf augmentations as varying $\xnoise$, which is label- and domain-independent:
\begin{align}
\label{eqn:augstd}
  \augstd(x)= \comb{\xobj,\augx{\xnoise},\xcore,\xspu},
\end{align}
where $\augx{\xnoise}$ is drawn from some augmentation distribution. 
\emph{Domain-invariant augmentations} $\auginv$ aim to randomize all \domaindependent features $\xcore$ and $\xspu$:
\begin{align}
\label{eqn:auginv}
  \auginv(x)= \comb{\xobj,\xnoise,\augx{\xcore},\augx{\xspu}},
\end{align}
where $\augx{\xcore},\augx{\xspu}$ are drawn from some distribution.
Finally, \emph{targeted augmentations} $\augtgt$ preserve $\xcore$ while aiming to randomize $\xspu$:
\begin{align}
\label{eqn:augtgt}
  \augtgt(x)= \comb{\xobj,\xnoise,\xcore,\augx{\xspu}},
\end{align}
where $\augx{\xspu}$ is drawn from some distribution.
Applying \shelf, domain-invariant, and targeted augmentations to the training distribution $\Ptrain$ yields new distributions over examples $\Pstd,\Pinv,$ and $\Ptgt$, respectively.

\tightparagraph{Training.}
Given $N$ training examples $\{(x^\ix, y^\ix)\}_{i=1}^N$ drawn from $\Ptrain$,
we learn a model that minimizes the average loss on the (augmented) training data:
\begin{align}
  &\thetahatnoaug = \argmin_\theta\E_{\Ptrainhat}\left[\ell(\theta; (x,y))\right]\\
  &\thetahataug = \argmin_\theta\E_{\Paughat}\left[\ell(\theta; (x,y))\right],
\end{align}
where $\Ptrainhat$ and $\Paughat$ are the empirical distributions over the unaugmented and augmented training data, respectively. 
The superscript $\mathsf{aug}$ can stand for $\mathsf{gen}$, $\mathsf{inv}$, or $\mathsf{tgt}$.

\subsection{Targeted augmentations for real-world datasets}
We instantiate targeted augmentations on real-world datasets from \refsec{datasets},
using domain knowledge about implicit $\xcore$ and $\xspu$ features.
Full details are in \refapp{app:augmentations}.

\tightparagraph{Species classification from camera trap images (\iwildcam).}
In \iwc, image backgrounds are \domaindependent features with both spurious and robust components.
While low-level background features are spurious, habitat features are robust.
\textbf{\cpsamey} transforms input $(x, y)$ by pasting the animal foreground onto a random training set background---but {only} onto backgrounds from training cameras that also observe $y$ (\reffig{augmentations}).
This randomizes low-level background features while roughly preserving habitat.
We use segmentation masks from~\citet{beery2021iwildcam}.

\tightparagraph{Tumor identification in histopathology slides (\camelyon).}
In \cam, stain color is a spurious \domaindependent feature, while stage-related features are robust \domaindependent features.
\textbf{\stainj} \citep{tellez2018whole} transforms $x$ by jittering its color in the hematoxylin and eosin staining color space (\reffig{augmentations}).
In contrast, domain-invariant augmentations can distort cell morphology to attain invariance, which loses information.

\tightparagraph{Bird species recognition from audio recordings (\birdcalls).}
In \birds, low-level noise and gain levels are spurious \domaindependent features, while habitat-specific noise is a robust \domaindependent feature.
\textbf{\cpsamer} leverages time-frequency bounding boxes to paste bird calls onto other training set recordings from the same geographic region (Southwestern Amazon Basin, Hawaii, or Northeastern United States) (\reffig{augmentations}).
After pasting the bird call, we also jitter hue levels of the spectrogram to simulate randomizing microphone gain settings.

\section{Analysis and simulations}\label{sec:toy}
We now motivate targeted augmentations and illustrate the shortcomings of \notcrossdomain and domain-invariant augmentations
in a simple linear regression setting, building off of the framework in \refsec{setup}.
To our knowledge, our analysis is the first to characterize how different augmentation strategies affect OOD risk and its scaling with the number of domains.
It also proposes a natural theoretical setting for OOD generalization, in which the distribution shift arises from finite-domain effects, departing from prior work that considers worst-case shifts \citep{rosenfeld2020risks,chen2021iterative}.

\subsection{Linear regression setting}\label{sec:toy_setup}

\tightparagraph{Data distribution.}
We model each domain $d$ as having latent attributes $\mud \defeq [\mucore, \muspu]$, which affect the distribution of the corresponding \domaindependent features $\xcore,\xspu$.
In \iwc, $\mucore$ intuitively corresponds to a habitat indicator and label prior.
In the linear setting, these domain attributes are drawn as 
\begin{equation}
\begin{aligned}
    \mucore \sim \mathcal N(0, \taucore^2 I)&\\
    \muspu \sim \mathcal N(0, \tauspu^2 I)&.
\end{aligned}
\label{eqn:toy_mu}
\end{equation}
The dimensionality of $\mud$ is $\dimdom$, and the dimensionality of $\mucore$ is $\dimcore$.
Following the feature decomposition in \reffig{datasets}, we consider inputs $x = [\xobj,\xnoise,\xcore,\xspu]$, \ie $\comb{\cdot}$ is a concatenation.
The training data is drawn uniformly from $\ndom$ training domains.
Within each domain, inputs $x$ are drawn according to the following distribution:
\begin{equation}
\begin{aligned}
    \xobj \sim \mathcal N(0, I)&\\
    \xnoise \sim \mathcal N(0, I)&\\
    \xcore\mid\dom \sim \mathcal N(\mucore, \sigmacore^2 I)&\\
    \xspu\mid\dom \sim \mathcal N(\muspu, \sigmaspu^2 I)&.
\end{aligned}
\label{eqn:toy_x}
\end{equation}
The \domaindependent features $\xcore$ and $\xspu$ are centered around the corresponding domain attributes $\mucore$ and $\muspu$, while the \domainindependent features $\xobj$ and $\xnoise$ are not.
We define the \varratioword $\varratio  \defeq \taucore^2 / \sigmacore^2$, which is the ratio of variances in $\mud$ and feature noise.
When $\varratio  > 1$, examples within a domain tend to be more similar to each other than to examples from other domains; we consider the typical setting in which $\varratio  > 1$.

The output $y \in \mathbb R$ is a linear function of both $\xobj$ and robust domain attribute $\mucore$:
\begin{equation}
\begin{aligned}
    y = \thetastarobj^\top\xobj + \thetastarcore^\top\mucore + \mathcal N(0, \sigmay^2).&
\end{aligned}
\label{eqn:toy_y}
\end{equation}
For convenience, we define the parameters for \domaindependent components as $\thetastardom\triangleq[\thetastarcore,\thetastarspu]$ where $\thetastarspu=0$.
Although $y$ depends on the domain attributes $\mud$, models cannot directly observe $\mud$, and instead only observe the noised features $\xcore,\xspu$.

The data generating process above tells us that in $\Pmeta$ \refeqn{pmeta}, $y$ and $\xspu$ are independent, as $\mucore$ and $\muspu$ are independent in distribution \refeqn{toy_mu}.
However, the training distribution $\Ptrain$ is generated from only a small, finite sample of $(\mucore, \muspu)$ pairs, one for each of the $\ndom$ training domains.
The smaller $\ndom$ is, the more correlated $\xspu$ appears with $\xcore$ (and thus $y$) in the training distribution.
This is true even with infinite examples per domain: 
so long as $\ndom$ is fixed, more training examples reveal that $\xcore \indep \xnoise$, but $\xcore$ will remain correlated with $\xspu$.
This \textit{finite-domain effect} enables models to infer $\mucore$ (and thus $y$) not only from $\xcore$, but also from $\xspu$.
Intuitively, models do this by memorizing $(\mucore, \muspu)$ pairs inferred from $(\xcore, \xspu)$ associations in the training distribution.
However, this strategy does not generalize OOD, since in $\Pmeta$, $\xspu$ is independent of $\mucore$. 

\tightparagraph{Augmentations.}
Recall from \refsec{augmentations} that \shelf, domain-invariant, and targeted augmentations replace components of $x$ with draws from an augmentation distribution. 
We preserve $y$ when augmenting and fix the augmentation distributions to preserve each feature's marginal distribution:
\begin{equation}
\begin{aligned}
  \augx{\xnoise} \sim \mathcal N(0, I)\\
  \augx{\xcore} \sim \mathcal N(0,(\sigma^2+\tau^2)I)\\
  \augx{\xspu} \sim \mathcal N(0,(\sigma^2+\tau^2)I).
\end{aligned}
\end{equation}

\tightparagraph{Models.}
We study linear models, specifically ordinary least squares in theoretical analysis (\refsec{theory}) and ridge regression in simulations (\refsec{simulation}).

\subsection{Theory}\label{sec:theory}
In this section, we first show that unaugmented models fail to generalize OOD when the domain generalization problem is \emph{underspecified} (\refthm{ood-bound-unaug}), \ie when there are fewer training domains than the dimensionality of the \domaindependent features, as is typically the case in real-world domain generalization problems.
This motivates targeted augmentations;
by eliminating spurious \domaindependent features, targeted augmentations bring the problem to a fully specified regime.
We prove that targeted augmentations improve OOD risk in expectation (\refthms{ood-bound-tgt-zhu}{gap-bound}), whereas \notcrossdomain and domain-invariant augmentations incur high OOD risk (\refcor{ood-bound-std} and \refthm{ood-inv}).

Our analysis assumes infinite data per domain, but finite training domains.
This allows us to focus on the effects of OOD generalization while simplifying traditional sample complexity issues, which are better understood.

\tightparagraph{Overview.} 
We study the expected excess OOD risk $\E\left[\Rood(\theta) - \Rood(\thetabest)\right]$, where the expectation is over random draws of training domains, and $\thetabest \defeq \argmin_\theta\Rood(\theta)$ is the oracle model that attains optimal performance in the population $\Pmeta$.
To show that targeted augmentations improve the expected OOD risk, we lower bound the expected excess risk for unaugmented models, upper bound it for models with targeted augmentations, and then demonstrate a gap between the two bounds.
Proofs are in \refapp{app:theory}.

\tightparagraph{Lower bound for excess OOD risk with no or \notcrossdomain augmentations.}
When the number of domains is smaller than the dimensionality of the \domaindependent features ($\ndom<\dimdom$), unaugmented models perform poorly OOD.\begin{restatable}[Excess OOD risk without augmentations]{thm}{oodboundunaug}
\label{thm:ood-bound-unaug}
If $\ndom<\dimdom$, the expected excess OOD risk of the unaugmented model is bounded below as
\begin{small}
\begin{align*}
\scalebox{0.95}{%
$\displaystyle
\E\left[\Rood(\thetahatnoaug) - \Rood(\thetabest)\right]\ge \frac{\tau^2\varratio \left\|\thetastarcore\right\|^2}{1+\varratio }\left(1-\frac{\ndom}{\dimdom}\right).
$%
}
\end{align*}
\end{small}
\end{restatable}
\begin{proofsketch}
The learned estimator has weights $\thetahatnoaugdom=(\sigma^2I+ M)^{-1}M\thetastardom$, where $M \defeq \frac1{\ndom} \sum_{\dom=1}^\ndom \mud \mudT$ is a random $\dimdom$-dimensional Wishart matrix.
As we only observe $\ndom < \dimdom$ training domains, $M$ is not full rank, with nullity $\dimdom - \ndom$.
We lower bound the overall excess risk by the excess risk incurred in the null space of $M$, which is $\frac{\tau^2\varratio }{1+\varratio }\sum_{i=1}^{\dimdom-\ndom}(u_i^\top\thetastardom)^2$;
each $u_i$ is an eigenvector with a zero eigenvalue and the summation term is thus the squared norm of a projection of $\thetastardom$ onto the null space of $M$.
In expectation, the squared norm is 
$||\thetastardom||^2(1-\frac{\ndom}{\dimdom})$
because $M$ has spherically symmetric eigenvectors. Finally, $\|\thetastardom\|=\|\thetastarcore\|$ because $\thetastarspu=0$.
\end{proofsketch}
To contextualize the bound, we discuss the relative scale of the excess OOD risk with respect to the OOD risk of the oracle model $\Rood(\thetabest) = \sigmaeps^2 + \tau^2\|\thetastarcore\|^2/(1+\varratio )$, where the first and second terms are from irreducible error in $y$ and feature noise in $\xcore$, respectively (\refprop{best-ood-error}).
The excess error of the unaugmented model is higher than the second term by a factor of $\varratio (1-\ndom/\dimdom),$ where $\varratio >1$ is the \varratioword and $\ndom$ is the number of domains.
Thus, in typical settings where $\ndom$ is small relative to $\dimdom$ and the \varratioword $\varratio $ is large, unaugmented models suffer substantial OOD error.

Models trained with \notcrossdomain augmentations have the same lower bound (\refcor{ood-bound-std} in \refapp{app:lowerbound}), as applying \notcrossdomain augmentations results in the same model as unaugmented training in the infinite data setting.
Our analysis captures the shortcomings of \notcrossdomain augmentations, which primarily improve sample complexity (not domain complexity);
as evident in the high OOD risk even in the infinite data setting, improving sample complexity alone fails to achieve OOD robustness.

\tightparagraph{Motivating targeted augmentations.}
The core problem above is \emph{underspecification},
in which the number of domains is smaller than the dimensionality of the \domaindependent features ($\ndom<\dimdom$);
there are fewer instances of $\mud$ than its dimensionality (although $\E[xx^\top]$ is full rank due to feature noise).
In such regimes, it is not possible to approximate $\thetastardom$ well, and models incur high OOD risk.
We can mitigate this via targeted augmentations, which randomizes the spurious \domaindependent feature.
This decreases the effective dimensionality from $\dimdom$ to $\dimcore$, the dimensionality of only the robust components, as models would no longer use the spurious feature.

\tightparagraph{Upper bound for excess OOD risk with targeted augmentations.}
With targeted augmentations, the problem (even without feature noise) is no longer underspecified when the number of training domains $\ndom$ is large enough relative to $\dimcore < \dimdom$.
In this fully specified regime, we can upper bound the expected excess OOD risk as $O(\log\ndom/\ndom)$.
This resembles the standard rates for random design linear regression up to a log factor \cite{hsu2011analysis,gyorfi2002distribution};
standard analysis shows that excess ID risk has a $O(1/N)$ convergence rate where $N$ is the number of samples, and we show that excess OOD risk has an analogous convergence rate as a function of the number of domains instead of examples.
\begin{restatable}[Excess OOD risk with targeted augmentations]{thm}{oodboundtgtzhu}
\label{thm:ood-bound-tgt-zhu}
Assume $\varratio >1$.
For any $0<r<1$ and large enough $\ndom$ such that $D>2(\dimcore+2)\log(4D\dimcore)/(1-r)^2$,
the excess OOD risk is bounded above as
\begin{align*}
  &\E\left[\Rood(\thetahattgt) - \Rood(\thetabest)\right]\\
  &\le \frac{\tau^2\varratio \left\|\thetastarcore\right\|^2}{1+\varratio }\left(\frac{1}{\ndom} + \frac{2\log(4\ndom\dimcore)(\dimcore+2)}{\ndom(1+\varratio r)^2}\right).\notag
\end{align*}
\end{restatable}
\begin{proofsketch}
The learned estimator has weights $\thetahattgtspu=0$ and $\thetahattgtcore=(\sigma^2I+\Mcore)^{-1}\Mcore\thetastarcore$, where $\Mcore \defeq \frac{1}{D}\sum_{d=1}^D\mucore\mucoreT$ is a random  $\dimcore$-dimensional Wishart matrix.
The excess risk can be written as $\sum_{i=1}^{\dimcore}\frac{\sigma^4(\tau^2-\lambda_i)^2}{(\sigma^2+\tau^2)(\lambda_i+\sigma^2)^2}(u_i^\top\thetastarcore)^2$, where $\lambda_i$ and $u_i$ are eigenvalues and eigenvectors of $\Mcore$, respectively.
Note that this excess risk is low when $\ndom$ is sufficiently large relative to $\dimcore$ such that the eigenvalues are sufficiently close to their expected value $\tau^2$.
We upper bound the excess OOD risk by applying concentration of measure arguments from \citet{zhu2012short} to the eigenvalues of $\Mcore$.
\end{proofsketch}
Compared to the lower bound for unaugmented models (\refthm{ood-bound-unaug}), this upper bound has qualitatively different behavior.
It depends on $\dimcore$ instead of $\dimdom$,
and it converges to 0 at a fast rate of $O(\log\ndom/\ndom)$ whereas the lowerbound is a negative linear function of the number of $\ndom$. 

\tightparagraph{Targeted augmentations improve expected OOD risk.}
We now combine the lower and upper bounds to show that targeted augmentations improve expected OOD risk.
\begin{restatable}[Targeted augmentations improve OOD risk]{thm}{gapbound}
  \label{thm:gap-bound}
  If $\varratio >1$ and $\dimcore$ is small relative to $\dimdom$ such that 
  \begin{align*}
    \dimcore< \frac{\dimdom}{\log(2\dimdom)}\cdot\frac{1}{4(1+\gamma^4/(\varratio -1)^2)},
  \end{align*}
  then for $\ndom$ such that
  \begin{align*}
    \ndom &>\frac{4\gamma^4}{(\varratio -1)^2}(\dimcore+2)\log(2\dimdom)\\
    \ndom &< \dimdom - 4(\dimcore+2)\log(2\dimdom),
  \end{align*}
  the improvement in expected OOD risk is positive:
  \begin{align*}
    \E&\left[\Rood(\thetahatnoaug) - \Rood(\thetahattgt)\right] > 0.
  \end{align*}
\end{restatable}
As expected, the minimum and maximum number of domains for which there is a provable gap is proportional to $\dimcore$ and $\dimdom$, respectively.
However, there is some looseness in the bound; in simulations (\refsec{simulation}), we see a substantial gap consistent with the above result, including for $\ndom$ outside the proven range.

\tightparagraph{Domain-invariant augmentations incur high OOD error.}
Finally, we show that domain-invariant augmentations incur high OOD risk in expectation.
\begin{restatable}[OOD error with domain-invariant augmentations]{thm}{oodinv}
\label{thm:ood-inv}
  For all $\ndom$, expected OOD risk is
  \begin{align*}
    \E[\Rood(\thetahatinv)-\Rood(\thetabest)] = \frac{\tau^2\varratio \|\thetastarcore\|^2}{1+\varratio }.
  \end{align*}
\end{restatable}
Because domain-invariant augmentations randomize all \domaindependent features, models do not use any \domaindependent features, including the robust components that are crucial for prediction.
As a result, the expected OOD risk is high (higher than the lower bound for unaugmented models in \refthm{ood-bound-unaug}), and the error does not decay with the number of domains $\ndom$.

\subsection{Simulations}\label{sec:simulation}
The analysis in \refsec{theory} assumes infinite data per domain.
We now present simulation results with finite data in a high-sample ($N=\num{100000}$) and low-sample ($N=\num{5000}$) regime,
where $N$ is the total number of examples across all domains.
We fix $\varratio  = 10, \dimcore = 5$ and $\dimspu = 500$. Additional details and results are in \refapp{app:simulations}.

\begin{figure}[tb]
    \centering
    \includegraphics[width=\linewidth]{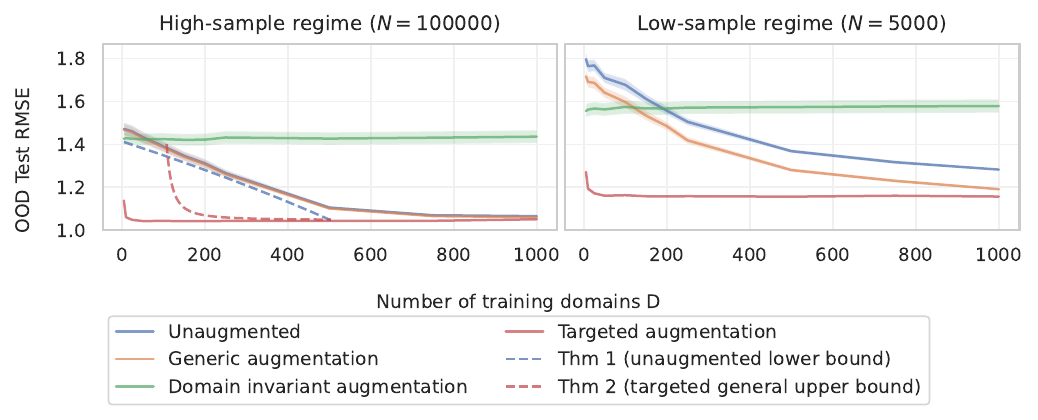}
    \caption{
        Targeted augmentations (red line) improve OOD error substantially, while \shelf (orange) or unaugmented (blue) models require many training domains to attain low OOD error.
        Domain-invariant augmentations (green line) have constant high error.
        We plot OOD RMSE for varying number of training domains, with standard errors over 10 random seeds. 
        We also plot the risk bounds from \refsec{theory} for the high-sample regime; because the bounds assume infinite data, we do not plot them for the low-sample case.
        The plotted \refthm{ood-bound-tgt-zhu} bound is a more general version (\refapp{app:proof-ood-bound-tgt-zhu}). 
        }
    \label{fig:simulations_ood}
    \vspace{-1.5em}
\end{figure}

\tightparagraph{High-sample regime ($N=\num{100000}$).}
In \reffig{simulations_ood} (left), we plot OOD RMSE against the number of training domains $\ndom$, together with our upper bound for targeted augmentations (a more general version of \refthm{ood-bound-tgt-zhu} in \refapp{app:theory})
and lower bound for unaugmented training (\refthm{ood-bound-unaug}).

We observe the trends suggested by our theory.
When $\ndom$ is small, the unaugmented model (blue) has high OOD error, and as $\ndom$ increases, OOD error slowly decays.
Training with \shelf augmentation (orange) does not improve over unaugmented training.
In contrast, training with targeted augmentation (red) significantly reduces OOD error.
There is a substantial gap between the red and orange/blue lines, which persists even when $\ndom$ is outside of the window guaranteed by \refthm{gap-bound}.
Finally, domain-invariant augmentations result in high OOD error (green) that does not decrease with increasing domains, as in \refthm{ood-inv}.

\tightparagraph{Low-sample regime ($N=\num{5000}$).}
In \reffig{simulations_ood} (right), we plot OOD RMSE against the number of training domains $\ndom$ when the sample size is small.
The unaugmented and targeted models follow the same trends as in the high-sample regime.
However, in the low-sample regime, \shelf augmentation \emph{does} reduce OOD error compared to the unaugmented model.
When the total number of examples $N$ is small, models are incentivized to memorize individual examples using $\xnoise$.
\Shelf augmentation prevents this behavior, resulting in an ID and OOD improvement over unaugmented training (also see \reffig{simulations_id} in \refapp{app:simulations}).
However, the OOD error of \shelf augmentation only decays slowly with $\ndom$  and is significantly higher than targeted augmentation for $D < \num{1000}$.
Domain-invariant augmentation results in a constant level of OOD error, which improves over the unaugmented and \shelf models for small values of $\ndom$, but underperforms once $\ndom$ is larger.

Overall, our simulations corroborate the theory and show that targeted augmentations offer significant OOD gains in the linear regression setting.
In contrast, \shelf and domain-invariant augmentations improve over unaugmented training only in the low-sample regime.

\section{Experiments on real-world datasets}\label{sec:experiments}
We return to the real-world datasets (\iwildcam, \camelyon, \birdcalls) and augmentations introduced in \refsec{datasets}, where
we compare targeted augmentations to unaugmented training, \shelf augmentations, and domain invariance baselines.
To approximate the overall distribution $\Pmeta$ \refeqn{pmeta}, we evaluate on held-out domains $\Domtest$, where $\Domtest \cap \Domtrain = \emptyset$.

\tightparagraph{\Shelf augmentations.} On image datasets \iwc and \cam, we compare to
RandAugment \citep{cubuk2020randaugment}, CutMix \citep{yun2019cutmix}, MixUp \citep{zhang2017mixup}, and Cutout \citep{devries2017improved}.
On audio dataset \birds, we compare to MixUp, SpecAugment \citep{park2019specaugment}, random low / high pass filters, and noise reduction via spectral gating \citep{noisereduce}.
Since the targeted augmentation for \birds (\cpsamer) includes color jitter as a subroutine, we also include a baseline of augmenting with only color jitter.

\tightparagraph{Domain invariance baselines.} We compare to LISA \citep{yao2022improving}, a data augmentation strategy that aims to encourage domain invariance by applying either MixUp or CutMix to inputs of the same class across domains. 
We also compare to other domain invariance algorithms that do not involve augmentation: (C)DANN \citep{long2018conditional,ganin2016domain}, DeepCORAL~\citep{sun2016deep,sun2017correlation}, and IRM~\citep{arjovsky2019invariant}.

Samples of the augmentations are shown in \reffig{augmentations}. 
Additional experimental details can be found in \refapp{app:hyperparams}. Code and \birdcalls are released at \href{https://github.com/i-gao/targeted-augs}{this link.}

\begin{figure*}[t]
    \centering
    \includegraphics[width=\textwidth]{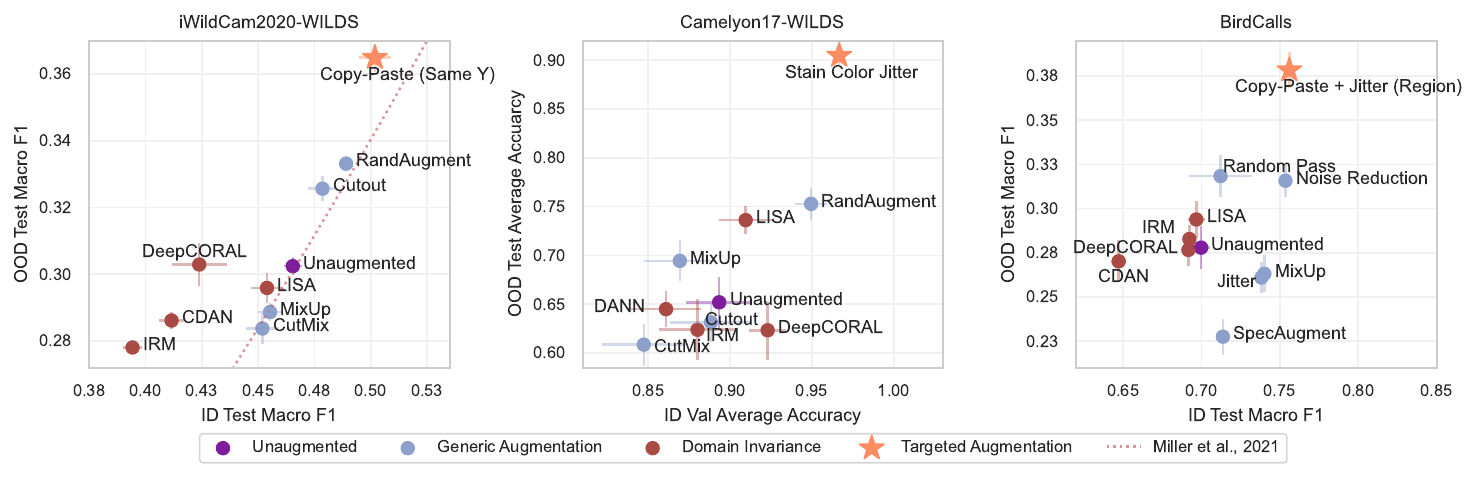}
    \caption{We plot the \id (ID) performance of methods against their \ood (OOD) performance. Error bars are standard errors over replicates. Targeted augmentations significantly improve OOD performance over the nearest baseline, improving OOD Macro F1 on \iwc from 33.3\% $\to$ 36.5\%, OOD average accuracy on \cam from 75.3\% $\to$ 90.5\%, and OOD Macro F1 on \birds from 31.8\% $\to$ 37.8\%. Tables and additional details can be found in \refapp{app:experiment}.}
    \label{fig:scatterplots}
    \vspace{0em}
\end{figure*}

\subsection{Results}\label{sec:results}
\reffig{scatterplots} plots the average ID versus OOD performance of each method.
On all three datasets, targeted augmentations significantly improve OOD performance.
Compared to the best-performing baseline, targeted augmentations improve OOD Macro F1 on \iwc from 33.3\% $\to$ 36.5\%, OOD average accuracy on \cam from 75.3\% $\to$ 90.5\%, and OOD Macro F1 on \birds from 31.8\% $\to$ 37.8\%.
On \iwc and \cam, which are part of the WILDS benchmark, these targeted augmentations set new state-of-the-art performances~\citep{koh2021wilds}.
\footnote{\birds is a new dataset, so targeted augmentations are state-of-the-art against the baselines reported here.}

Several \shelf augmentations were also able to improve OOD performance, although by smaller amounts than targeted augmentations; this matches our simulations in the low-sample regime in \refsec{simulation}.
RandAugment~\citep{cubuk2020randaugment} performs strongly on \iwc and \cam, and both noise reduction and random high / low pass filters perform well on \birds.
Some \shelf augmentations degraded performance (MixUp, CutMix, and SpecAugment), which may reflect the fact that these augmentations can also distort $\xobj$ and $\xcore$, \eg by mixing cell morphologies on \cam.

\tightparagraph{Effective robustness.} On \iwc, \citet{miller2021accuracy} showed that the ID and OOD performances of models across a range of sizes are linearly correlated; we plot their linear fit on \reffig{scatterplots} (left).
We found that our targeted augmentation \cpsamey confers what \citet{miller2021accuracy} termed \emph{effective robustness}, which is represented in the plot by a vertical offset from the line.
In contrast, \shelf augmentations improve OOD performance along the plotted line.
While the domain invariance methods also show effective robustness, they mostly underperform the unaugmented model in raw performance numbers.

Although neither \cam nor \birds have associated linear fits, we observe similar trends in \reffig{scatterplots}, with targeted augmentations offering significant OOD gains even at similar ID performances as other methods.

\tightparagraph{Ablation on $\xcore$.}\label{sec:ablations}
To demonstrate the importance of preserving $\xcore$, we modified the targeted augmentations for \iwc and \birds to be non-selective.
On \iwc, \cpsamey selectively pastes animal foregrounds onto backgrounds from domains which also observe $y$ in the training set; as an ablation, we studied \cpall, which draws backgrounds from all training domains, including cameras in which $y$ was not observed.
Similarly, on \birds, \cpsamer only pastes calls onto recordings from the original microphone's region; as an ablation, we studied \cpallr, which merges recordings indiscriminately.
These modified augmentations fail to preserve habitat features $\xcore$.
In \reftab{habitat_ablation}, we see that preserving $\xcore$ is important---compared to their targeted variants, the modified augmentations decrease OOD performance by 1.8\% on \iwc and 4.1\% on \birds.

\begin{table}[tb]
\centering
\caption{Randomizing habitat features in \iwildcam and \birdcalls degrades performance.}
\label{tab:habitat_ablation}
\resizebox{\linewidth}{!}{%
\begin{tabular}{llll}
\toprule
Dataset & Method & ID Test Macro F1 & OOD Test Macro F1 \\
\midrule
\multirow{3}{*}{\iwc} & Unaugmented & 46.5 (0.4) & 30.2 (0.3) \\
& Copy-Paste (All Backgrounds) & 47.1 (1.1) & 34.7 (0.5) \\
& Copy-Paste (Same Y) & \textbf{50.2 (0.7)} & \textbf{36.5 (0.4)} \\
\midrule
\multirow{3}{*}{\birds} & Unaugmented & 70.0 (0.5) & 27.8 (1.2) \\
& Copy-Paste + Jitter (All Regions) & \textbf{76.0 (0.3)} & 33.7 (1.0) \\
& Copy-Paste + Jitter (Same Region) & 75.6 (0.3) & \textbf{37.8 (1.0)} \\
\bottomrule
\end{tabular}%
}
\vspace{-2.2em}
\end{table}

\begin{table}[tb]
    \centering
    \caption{Finetuning CLIP ViT-L/14 with targeted augmentations improves OOD performance on \camelyon (accuracy) and \iwildcam (macro F1). Results averaged over 5 seeds with standard errors.}
    \label{tab:clip}
    \resizebox{\linewidth}{!}{%
    \begin{tabular}{llll}
    \toprule
    Dataset & Method & ID Performance & OOD Performance \\
    \midrule
    \multirow{2}{*}{\cam} & Unaugmented & \textbf{99.5 (0.0)} & 96.0 (0.2) \\
    & \stainj & 99.4 (0.0) & \textbf{97.1 (0.0)} \\
    \midrule
    \multirow{3}{*}{\iwc} & Unaugmented & 55.6 (0.8) & 43.5 (0.7) \\
    & Copy-Paste (Same Y) & \textbf{56.6 (0.7)} & \textbf{45.5 (0.3)} \\
    \bottomrule
    \end{tabular}%
    }
    \vspace{-1em}
\end{table}

\tightparagraph{Targeted augmentations improve OOD performance when finetuning CLIP.}
We also applied our targeted augmentations to CLIP ViT-L/14 \citep{radford2021learning}, a large-scale vision-language model (\reftab{clip}).
Targeted augmentations offer 1.1\% and 2\% OOD average gains over unaugmented finetuning on \iwc and \cam.

\section{Related work}\label{sec:related_work}
Additional related work is found in Appendix \ref{sec:app:related_work}.

\tightparagraph{Data augmentations for OOD robustness.}
Prior work has shown that \shelf augmentations designed for ID performance can improve OOD performance, but this effect is inconsistent across datasets \cite{gulrajani2020search,hendrycks2021many,wiles2021fine}.
Other work has sought to design augmentations specifically for robustness; these are often inspired by domain invariance and aim to randomize all \domaindependent features, including robust features $\xcore$ \cite{wang2020heterogeneous,xu2020adversarial,yan2020improve,yao2022improving}.
In contrast, we preserve $\xcore$ in targeted augmentations.

\tightparagraph{Analysis on data augmentations and domain generalization.}
Existing work usually analyzes augmentations in the standard \iid~setting \citep{dao2019kernel,he2019data,chen2020group,lyle2020benefits}, where augmentations improve sample complexity and reduce variance.
We instead analyze the effect of data augmentation on OOD performance.
There is limited theoretical work in this setting: 
\citet{ilse2021selecting} use augmentations to simulate interventions on domains, and \citet{wang2022out} show that one can recover a causal model given a set of augmentations encoding the relevant invariances.
These works are part of a broader thread of analysis which emphasizes robustness to \emph{worst-case} domain shifts; the aim is thus to recover models that only rely on causal features.
In contrast, we seek to generalize to unseen domains \emph{on average}.
Our analysis is related to work on meta-learning \citep{chen2021generalization,jose2021information}; however, these analyses focus on adaptation to new tasks instead of \ood generalization.

\tightparagraph{Failures of domain invariance.}
To improve OOD robustness, the domain invariance literature focuses on learning models which are invariant to \domaindependent features, such that representations are independent of domain marginally~\cite{ganin2016domain,albuquerque2019generalizing}.
Several works have pointed out failure modes of this approach, including \citet{mahajan2021domain}, who focus on cases where the distribution of causal features vary across domains; we additionally allow for $\xcore$ to be non-causal, \eg habitat features in \iwc and \birds.
\section{Conclusion}\label{sec:discussion}
We studied targeted augmentations, which randomize spurious \domaindependent features while preserving robust ones, 
and showed that they can significantly improve OOD performance over \shelf and domain-invariant augmentations.
These results illustrate that when the \ood generalization problem is underspecified, prior knowledge can provide additional structure and make the \ood generalization problem more tractable.
Future work could also explore methods for learning, rather than hand-designing, targeted augmentations; 
such approaches could leverage high-level prior knowledge on $\xcore$, or directly infer $\xcore$ from the training domains.

\FloatBarrier
\section*{Acknowledgements}
We are grateful to Henrik Marklund, Holger Klinck, and Sara Beery for their advice. This work was supported by NSF Frontier and Open Philanthropy awards. Shiori Sagawa was supported by the Apple Scholars in AI/ML PhD Fellowship. 

\bibliography{references}

\begin{thebibliography}{60}
\providecommand{\natexlab}[1]{#1}
\providecommand{\url}[1]{\texttt{#1}}
\expandafter\ifx\csname urlstyle\endcsname\relax
  \providecommand{\doi}[1]{doi: #1}\else
  \providecommand{\doi}{doi: \begingroup \urlstyle{rm}\Url}\fi

\bibitem[Albuquerque et~al.(2019)Albuquerque, Monteiro, Darvishi, Falk, and
  Mitliagkas]{albuquerque2019generalizing}
Albuquerque, I., Monteiro, J., Darvishi, M., Falk, T.~H., and Mitliagkas, I.
\newblock Generalizing to unseen domains via distribution matching.
\newblock \emph{arXiv preprint arXiv:1911.00804}, 2019.

\bibitem[Arjovsky et~al.(2019)Arjovsky, Bottou, Gulrajani, and
  Lopez-Paz]{arjovsky2019invariant}
Arjovsky, M., Bottou, L., Gulrajani, I., and Lopez-Paz, D.
\newblock Invariant risk minimization.
\newblock \emph{arXiv preprint arXiv:1907.02893}, 2019.

\bibitem[Bandi et~al.(2018)Bandi, Geessink, Manson, Van~Dijk, Balkenhol,
  Hermsen, Bejnordi, Lee, Paeng, Zhong, et~al.]{bandi2018detection}
Bandi, P., Geessink, O., Manson, Q., Van~Dijk, M., Balkenhol, M., Hermsen, M.,
  Bejnordi, B.~E., Lee, B., Paeng, K., Zhong, A., et~al.
\newblock From detection of individual metastases to classification of lymph
  node status at the patient level: the camelyon17 challenge.
\newblock \emph{IEEE transactions on medical imaging}, 38\penalty0
  (2):\penalty0 550--560, 2018.

\bibitem[Beery et~al.(2018)Beery, Van~Horn, and Perona]{beery2018recognition}
Beery, S., Van~Horn, G., and Perona, P.
\newblock Recognition in terra incognita.
\newblock In \emph{Proceedings of the European conference on computer vision
  (ECCV)}, pp.\  456--473, 2018.

\bibitem[Beery et~al.(2019)Beery, Morris, and Yang]{beery2019megadetector}
Beery, S., Morris, D., and Yang, S.
\newblock Efficient pipeline for camera trap image review.
\newblock \emph{arXiv preprint arXiv:1907.06772}, 2019.

\bibitem[Beery et~al.(2020)Beery, Liu, Morris, Piavis, Kapoor, Joshi, Meister,
  and Perona]{beery2020synthetic}
Beery, S., Liu, Y., Morris, D., Piavis, J., Kapoor, A., Joshi, N., Meister, M.,
  and Perona, P.
\newblock Synthetic examples improve generalization for rare classes.
\newblock In \emph{Proceedings of the IEEE/CVF Winter Conference on
  Applications of Computer Vision}, pp.\  863--873, 2020.

\bibitem[Beery et~al.(2021)Beery, Agarwal, Cole, and
  Birodkar]{beery2021iwildcam}
Beery, S., Agarwal, A., Cole, E., and Birodkar, V.
\newblock The iwildcam 2021 competition dataset.
\newblock \emph{arXiv preprint arXiv:2105.03494}, 2021.

\bibitem[Birodkar et~al.(2021)Birodkar, Lu, Li, Rathod, and
  Huang]{birodkar2021deepmac}
Birodkar, V., Lu, Z., Li, S., Rathod, V., and Huang, J.
\newblock The surprising impact of mask-head architecture on novel class
  segmentation.
\newblock In \emph{Proceedings of the IEEE/CVF International Conference on
  Computer Vision}, pp.\  7015--7025, 2021.

\bibitem[Chen et~al.(2021{\natexlab{a}})Chen, Shui, and
  Marchand]{chen2021generalization}
Chen, Q., Shui, C., and Marchand, M.
\newblock Generalization bounds for meta-learning: An information-theoretic
  analysis.
\newblock \emph{Advances in Neural Information Processing Systems},
  34:\penalty0 25878--25890, 2021{\natexlab{a}}.

\bibitem[Chen et~al.(2020)Chen, Dobriban, and Lee]{chen2020group}
Chen, S., Dobriban, E., and Lee, J.~H.
\newblock A group-theoretic framework for data augmentation.
\newblock \emph{The Journal of Machine Learning Research}, 21\penalty0
  (1):\penalty0 9885--9955, 2020.

\bibitem[Chen et~al.(2021{\natexlab{b}})Chen, Rosenfeld, Sellke, Ma, and
  Risteski]{chen2021iterative}
Chen, Y., Rosenfeld, E., Sellke, M., Ma, T., and Risteski, A.
\newblock Iterative feature matching: Toward provable domain generalization
  with logarithmic environments.
\newblock \emph{arXiv preprint arXiv:2106.09913}, 2021{\natexlab{b}}.

\bibitem[Cubuk et~al.(2020)Cubuk, Zoph, Shlens, and Le]{cubuk2020randaugment}
Cubuk, E.~D., Zoph, B., Shlens, J., and Le, Q.~V.
\newblock Randaugment: Practical automated data augmentation with a reduced
  search space.
\newblock In \emph{Proceedings of the IEEE/CVF conference on computer vision
  and pattern recognition workshops}, pp.\  702--703, 2020.

\bibitem[D'Amour et~al.(2020)D'Amour, Heller, Moldovan, Adlam, Alipanahi,
  Beutel, Chen, Deaton, Eisenstein, Hoffman, et~al.]{d2020underspecification}
D'Amour, A., Heller, K., Moldovan, D., Adlam, B., Alipanahi, B., Beutel, A.,
  Chen, C., Deaton, J., Eisenstein, J., Hoffman, M.~D., et~al.
\newblock Underspecification presents challenges for credibility in modern
  machine learning.
\newblock \emph{Journal of Machine Learning Research}, 2020.

\bibitem[Dao et~al.(2019)Dao, Gu, Ratner, Smith, De~Sa, and
  R{\'e}]{dao2019kernel}
Dao, T., Gu, A., Ratner, A., Smith, V., De~Sa, C., and R{\'e}, C.
\newblock A kernel theory of modern data augmentation.
\newblock In \emph{International Conference on Machine Learning}, pp.\
  1528--1537. PMLR, 2019.

\bibitem[Denton et~al.(2022)Denton, Wisdom, and Hershey]{denton2022improving}
Denton, T., Wisdom, S., and Hershey, J.~R.
\newblock Improving bird classification with unsupervised sound separation.
\newblock In \emph{ICASSP 2022-2022 IEEE International Conference on Acoustics,
  Speech and Signal Processing (ICASSP)}, pp.\  636--640. IEEE, 2022.

\bibitem[DeVries \& Taylor(2017)DeVries and Taylor]{devries2017improved}
DeVries, T. and Taylor, G.~W.
\newblock Improved regularization of convolutional neural networks with cutout.
\newblock \emph{arXiv preprint arXiv:1708.04552}, 2017.

\bibitem[Ganin et~al.(2016)Ganin, Ustinova, Ajakan, Germain, Larochelle,
  Laviolette, Marchand, and Lempitsky]{ganin2016domain}
Ganin, Y., Ustinova, E., Ajakan, H., Germain, P., Larochelle, H., Laviolette,
  F., Marchand, M., and Lempitsky, V.
\newblock Domain-adversarial training of neural networks.
\newblock \emph{The journal of machine learning research}, 17\penalty0
  (1):\penalty0 2096--2030, 2016.

\bibitem[Gontijo-Lopes et~al.(2020)Gontijo-Lopes, Smullin, Cubuk, and
  Dyer]{gontijo2020affinity}
Gontijo-Lopes, R., Smullin, S.~J., Cubuk, E.~D., and Dyer, E.
\newblock Affinity and diversity: Quantifying mechanisms of data augmentation.
\newblock \emph{arXiv preprint arXiv:2002.08973}, 2020.

\bibitem[Gulrajani \& Lopez-Paz(2020)Gulrajani and
  Lopez-Paz]{gulrajani2020search}
Gulrajani, I. and Lopez-Paz, D.
\newblock In search of lost domain generalization.
\newblock \emph{arXiv preprint arXiv:2007.01434}, 2020.

\bibitem[Gy{\"o}rfi et~al.(2002)Gy{\"o}rfi, Kohler, Krzyzak, Walk,
  et~al.]{gyorfi2002distribution}
Gy{\"o}rfi, L., Kohler, M., Krzyzak, A., Walk, H., et~al.
\newblock \emph{A distribution-free theory of nonparametric regression},
  volume~1.
\newblock Springer, 2002.

\bibitem[He et~al.(2019)He, Xie, Chen, Zhang, Wang, and Tian]{he2019data}
He, Z., Xie, L., Chen, X., Zhang, Y., Wang, Y., and Tian, Q.
\newblock Data augmentation revisited: Rethinking the distribution gap between
  clean and augmented data.
\newblock \emph{arXiv preprint arXiv:1909.09148}, 2019.

\bibitem[Hendrycks et~al.(2021)Hendrycks, Basart, Mu, Kadavath, Wang, Dorundo,
  Desai, Zhu, Parajuli, Guo, et~al.]{hendrycks2021many}
Hendrycks, D., Basart, S., Mu, N., Kadavath, S., Wang, F., Dorundo, E., Desai,
  R., Zhu, T., Parajuli, S., Guo, M., et~al.
\newblock The many faces of robustness: A critical analysis of
  out-of-distribution generalization.
\newblock In \emph{Proceedings of the IEEE/CVF International Conference on
  Computer Vision}, pp.\  8340--8349, 2021.

\bibitem[Hoffman et~al.(2018)Hoffman, Tzeng, Park, Zhu, Isola, Saenko, Efros,
  and Darrell]{hoffman2018cycada}
Hoffman, J., Tzeng, E., Park, T., Zhu, J.-Y., Isola, P., Saenko, K., Efros, A.,
  and Darrell, T.
\newblock Cycada: Cycle-consistent adversarial domain adaptation.
\newblock In \emph{International conference on machine learning}, pp.\
  1989--1998. Pmlr, 2018.

\bibitem[Hopping et~al.(2022)Hopping, Kahl, and
  Klinck]{w_alexander_hopping_2022_7079124}
Hopping, W.~A., Kahl, S., and Klinck, H.
\newblock {A collection of fully-annotated soundscape recordings from the
  Southwestern Amazon Basin}, September 2022.
\newblock URL \url{https://doi.org/10.5281/zenodo.7079124}.

\bibitem[Hsu et~al.(2011)Hsu, Kakade, and Zhang]{hsu2011analysis}
Hsu, D., Kakade, S.~M., and Zhang, T.
\newblock An analysis of random design linear regression.
\newblock \emph{arXiv preprint arXiv:1106.2363}, 2011.

\bibitem[Ilse et~al.(2021)Ilse, Tomczak, and Forr{\'e}]{ilse2021selecting}
Ilse, M., Tomczak, J.~M., and Forr{\'e}, P.
\newblock Selecting data augmentation for simulating interventions.
\newblock In \emph{International Conference on Machine Learning}, pp.\
  4555--4562. PMLR, 2021.

\bibitem[Joly et~al.(2021)Joly, Go{\"e}au, Kahl, Picek, Lorieul, Cole, Deneu,
  Servajean, Durso, Bolon, et~al.]{joly2021overview}
Joly, A., Go{\"e}au, H., Kahl, S., Picek, L., Lorieul, T., Cole, E., Deneu, B.,
  Servajean, M., Durso, A., Bolon, I., et~al.
\newblock Overview of lifeclef 2021: an evaluation of machine-learning based
  species identification and species distribution prediction.
\newblock In \emph{Experimental IR Meets Multilinguality, Multimodality, and
  Interaction: 12th International Conference of the CLEF Association, CLEF
  2021, Virtual Event, September 21--24, 2021, Proceedings}, pp.\  371--393.
  Springer, 2021.

\bibitem[Jose \& Simeone(2021)Jose and Simeone]{jose2021information}
Jose, S.~T. and Simeone, O.
\newblock Information-theoretic generalization bounds for meta-learning and
  applications.
\newblock \emph{Entropy}, 23\penalty0 (1):\penalty0 126, 2021.

\bibitem[Kahl et~al.(2022)Kahl, Charif, and Klinck]{stefan_kahl_2022_7079380}
Kahl, S., Charif, R., and Klinck, H.
\newblock {A collection of fully-annotated soundscape recordings from the
  Northeastern United States}, August 2022.
\newblock URL \url{https://doi.org/10.5281/zenodo.7079380}.

\bibitem[Koh et~al.(2021)Koh, Sagawa, Marklund, Xie, Zhang, Balsubramani, Hu,
  Yasunaga, Phillips, Gao, et~al.]{koh2021wilds}
Koh, P.~W., Sagawa, S., Marklund, H., Xie, S.~M., Zhang, M., Balsubramani, A.,
  Hu, W., Yasunaga, M., Phillips, R.~L., Gao, I., et~al.
\newblock Wilds: A benchmark of in-the-wild distribution shifts.
\newblock In \emph{International Conference on Machine Learning}, pp.\
  5637--5664. PMLR, 2021.

\bibitem[Kumar et~al.(2022)Kumar, Shen, Bubeck, and Gunasekar]{kumar2022fine}
Kumar, A., Shen, R., Bubeck, S., and Gunasekar, S.
\newblock How to fine-tune vision models with sgd.
\newblock \emph{arXiv preprint arXiv:2211.09359}, 2022.

\bibitem[Long et~al.(2018)Long, Cao, Wang, and Jordan]{long2018conditional}
Long, M., Cao, Z., Wang, J., and Jordan, M.~I.
\newblock Conditional adversarial domain adaptation.
\newblock \emph{Advances in neural information processing systems}, 31, 2018.

\bibitem[Lyle et~al.(2020)Lyle, van~der Wilk, Kwiatkowska, Gal, and
  Bloem-Reddy]{lyle2020benefits}
Lyle, C., van~der Wilk, M., Kwiatkowska, M., Gal, Y., and Bloem-Reddy, B.
\newblock On the benefits of invariance in neural networks.
\newblock \emph{arXiv preprint arXiv:2005.00178}, 2020.

\bibitem[Mahajan et~al.(2021)Mahajan, Tople, and Sharma]{mahajan2021domain}
Mahajan, D., Tople, S., and Sharma, A.
\newblock Domain generalization using causal matching.
\newblock In \emph{International Conference on Machine Learning}, pp.\
  7313--7324. PMLR, 2021.

\bibitem[Miller et~al.(2021)Miller, Taori, Raghunathan, Sagawa, Koh, Shankar,
  Liang, Carmon, and Schmidt]{miller2021accuracy}
Miller, J.~P., Taori, R., Raghunathan, A., Sagawa, S., Koh, P.~W., Shankar, V.,
  Liang, P., Carmon, Y., and Schmidt, L.
\newblock Accuracy on the line: on the strong correlation between
  out-of-distribution and in-distribution generalization.
\newblock In \emph{International Conference on Machine Learning}, pp.\
  7721--7735. PMLR, 2021.

\bibitem[Navine et~al.(2022)Navine, Kahl, Tanimoto-Johnson, Klinck, and
  Hart]{amanda_navine_2022_7078499}
Navine, A., Kahl, S., Tanimoto-Johnson, A., Klinck, H., and Hart, P.
\newblock {A collection of fully-annotated soundscape recordings from the
  Island of Hawai'i}, September 2022.
\newblock URL \url{https://doi.org/10.5281/zenodo.7078499}.

\bibitem[Park et~al.(2019)Park, Chan, Zhang, Chiu, Zoph, Cubuk, and
  Le]{park2019specaugment}
Park, D.~S., Chan, W., Zhang, Y., Chiu, C.-C., Zoph, B., Cubuk, E.~D., and Le,
  Q.~V.
\newblock Specaugment: A simple data augmentation method for automatic speech
  recognition.
\newblock \emph{arXiv preprint arXiv:1904.08779}, 2019.

\bibitem[Puli et~al.(2022)Puli, Joshi, He, and Ranganath]{puli2022nuisances}
Puli, A., Joshi, N., He, H., and Ranganath, R.
\newblock Nuisances via negativa: Adjusting for spurious correlations via data
  augmentation.
\newblock \emph{arXiv preprint arXiv:2210.01302}, 2022.

\bibitem[Radford et~al.(2021)Radford, Kim, Hallacy, Ramesh, Goh, Agarwal,
  Sastry, Askell, Mishkin, Clark, et~al.]{radford2021learning}
Radford, A., Kim, J.~W., Hallacy, C., Ramesh, A., Goh, G., Agarwal, S., Sastry,
  G., Askell, A., Mishkin, P., Clark, J., et~al.
\newblock Learning transferable visual models from natural language
  supervision.
\newblock In \emph{International conference on machine learning}, pp.\
  8748--8763. PMLR, 2021.

\bibitem[Robey et~al.(2021)Robey, Pappas, and Hassani]{robey2021model}
Robey, A., Pappas, G.~J., and Hassani, H.
\newblock Model-based domain generalization.
\newblock \emph{Advances in Neural Information Processing Systems},
  34:\penalty0 20210--20229, 2021.

\bibitem[Rosenfeld et~al.(2020)Rosenfeld, Ravikumar, and
  Risteski]{rosenfeld2020risks}
Rosenfeld, E., Ravikumar, P., and Risteski, A.
\newblock The risks of invariant risk minimization.
\newblock \emph{arXiv preprint arXiv:2010.05761}, 2020.

\bibitem[Ruifrok et~al.(2001)Ruifrok, Johnston,
  et~al.]{ruifrok2001quantification}
Ruifrok, A.~C., Johnston, D.~A., et~al.
\newblock Quantification of histochemical staining by color deconvolution.
\newblock \emph{Analytical and quantitative cytology and histology},
  23\penalty0 (4):\penalty0 291--299, 2001.

\bibitem[Sagawa et~al.(2021)Sagawa, Koh, Lee, Gao, Xie, Shen, Kumar, Hu,
  Yasunaga, Marklund, et~al.]{sagawa2021extending}
Sagawa, S., Koh, P.~W., Lee, T., Gao, I., Xie, S.~M., Shen, K., Kumar, A., Hu,
  W., Yasunaga, M., Marklund, H., et~al.
\newblock Extending the wilds benchmark for unsupervised adaptation.
\newblock \emph{arXiv preprint arXiv:2112.05090}, 2021.

\bibitem[Sainburg(2022)]{noisereduce}
Sainburg, T.
\newblock Noise reduction in python using spectral gating, 2022.
\newblock URL \url{https://github.com/timsainb/noisereduce}.

\bibitem[Sun \& Saenko(2016)Sun and Saenko]{sun2016deep}
Sun, B. and Saenko, K.
\newblock Deep coral: Correlation alignment for deep domain adaptation.
\newblock In \emph{European conference on computer vision}, pp.\  443--450.
  Springer, 2016.

\bibitem[Sun et~al.(2017)Sun, Feng, and Saenko]{sun2017correlation}
Sun, B., Feng, J., and Saenko, K.
\newblock Correlation alignment for unsupervised domain adaptation.
\newblock In \emph{Domain Adaptation in Computer Vision Applications}, pp.\
  153--171. Springer, 2017.

\bibitem[Tellez et~al.(2018)Tellez, Balkenhol, Otte-H{\"o}ller, van~de Loo,
  Vogels, Bult, Wauters, Vreuls, Mol, Karssemeijer, et~al.]{tellez2018whole}
Tellez, D., Balkenhol, M., Otte-H{\"o}ller, I., van~de Loo, R., Vogels, R.,
  Bult, P., Wauters, C., Vreuls, W., Mol, S., Karssemeijer, N., et~al.
\newblock Whole-slide mitosis detection in h\&e breast histology using phh3 as
  a reference to train distilled stain-invariant convolutional networks.
\newblock \emph{IEEE transactions on medical imaging}, 37\penalty0
  (9):\penalty0 2126--2136, 2018.

\bibitem[Tellez et~al.(2019)Tellez, Litjens, B{\'a}ndi, Bulten, Bokhorst,
  Ciompi, and van~der Laak]{tellez2019quantifying}
Tellez, D., Litjens, G., B{\'a}ndi, P., Bulten, W., Bokhorst, J.-M., Ciompi,
  F., and van~der Laak, J.
\newblock Quantifying the effects of data augmentation and stain color
  normalization in convolutional neural networks for computational pathology.
\newblock \emph{Medical image analysis}, 58, 2019.

\bibitem[Wang et~al.(2022)Wang, Yi, Chen, and Zhu]{wang2022out}
Wang, R., Yi, M., Chen, Z., and Zhu, S.
\newblock Out-of-distribution generalization with causal invariant
  transformations.
\newblock In \emph{Proceedings of the IEEE/CVF Conference on Computer Vision
  and Pattern Recognition}, pp.\  375--385, 2022.

\bibitem[Wang et~al.(2020)Wang, Li, and Kot]{wang2020heterogeneous}
Wang, Y., Li, H., and Kot, A.~C.
\newblock Heterogeneous domain generalization via domain mixup.
\newblock In \emph{ICASSP 2020-2020 IEEE International Conference on Acoustics,
  Speech and Signal Processing (ICASSP)}, pp.\  3622--3626. IEEE, 2020.

\bibitem[Wiles et~al.(2021)Wiles, Gowal, Stimberg, Rebuffi, Ktena, Dvijotham,
  and Cemgil]{wiles2021fine}
Wiles, O., Gowal, S., Stimberg, F., Rebuffi, S.-A., Ktena, I., Dvijotham,
  K.~D., and Cemgil, A.~T.
\newblock A fine-grained analysis on distribution shift.
\newblock In \emph{International Conference on Learning Representations}, 2021.

\bibitem[Wortsman et~al.(2022)Wortsman, Ilharco, Gadre, Roelofs, Gontijo-Lopes,
  Morcos, Namkoong, Farhadi, Carmon, Kornblith, et~al.]{wortsman2022model}
Wortsman, M., Ilharco, G., Gadre, S.~Y., Roelofs, R., Gontijo-Lopes, R.,
  Morcos, A.~S., Namkoong, H., Farhadi, A., Carmon, Y., Kornblith, S., et~al.
\newblock Model soups: averaging weights of multiple fine-tuned models improves
  accuracy without increasing inference time.
\newblock In \emph{International Conference on Machine Learning}, pp.\
  23965--23998. PMLR, 2022.

\bibitem[Xu et~al.(2020)Xu, Zhang, Ni, Li, Wang, Tian, and
  Zhang]{xu2020adversarial}
Xu, M., Zhang, J., Ni, B., Li, T., Wang, C., Tian, Q., and Zhang, W.
\newblock Adversarial domain adaptation with domain mixup.
\newblock In \emph{Proceedings of the AAAI Conference on Artificial
  Intelligence}, volume~34, pp.\  6502--6509, 2020.

\bibitem[Yan et~al.(2020)Yan, Song, Li, Zou, and Ren]{yan2020improve}
Yan, S., Song, H., Li, N., Zou, L., and Ren, L.
\newblock Improve unsupervised domain adaptation with mixup training.
\newblock \emph{arXiv preprint arXiv:2001.00677}, 2020.

\bibitem[Yao et~al.(2022)Yao, Wang, Li, Zhang, Liang, Zou, and
  Finn]{yao2022improving}
Yao, H., Wang, Y., Li, S., Zhang, L., Liang, W., Zou, J., and Finn, C.
\newblock Improving out-of-distribution robustness via selective augmentation.
\newblock \emph{arXiv preprint arXiv:2201.00299}, 2022.

\bibitem[Yun et~al.(2019)Yun, Han, Oh, Chun, Choe, and Yoo]{yun2019cutmix}
Yun, S., Han, D., Oh, S.~J., Chun, S., Choe, J., and Yoo, Y.
\newblock Cutmix: Regularization strategy to train strong classifiers with
  localizable features.
\newblock In \emph{Proceedings of the IEEE/CVF international conference on
  computer vision}, pp.\  6023--6032, 2019.

\bibitem[Zhang et~al.(2017)Zhang, Cisse, Dauphin, and
  Lopez-Paz]{zhang2017mixup}
Zhang, H., Cisse, M., Dauphin, Y.~N., and Lopez-Paz, D.
\newblock mixup: Beyond empirical risk minimization.
\newblock \emph{arXiv preprint arXiv:1710.09412}, 2017.

\bibitem[Zhou et~al.(2020{\natexlab{a}})Zhou, Yang, Hospedales, and
  Xiang]{zhou2020deep}
Zhou, K., Yang, Y., Hospedales, T., and Xiang, T.
\newblock Deep domain-adversarial image generation for domain generalisation.
\newblock In \emph{Proceedings of the AAAI Conference on Artificial
  Intelligence}, volume~34, pp.\  13025--13032, 2020{\natexlab{a}}.

\bibitem[Zhou et~al.(2020{\natexlab{b}})Zhou, Yang, Hospedales, and
  Xiang]{zhou2020learning}
Zhou, K., Yang, Y., Hospedales, T., and Xiang, T.
\newblock Learning to generate novel domains for domain generalization.
\newblock In \emph{European conference on computer vision}, pp.\  561--578.
  Springer, 2020{\natexlab{b}}.

\bibitem[Zhu(2012)]{zhu2012short}
Zhu, S.
\newblock A short note on the tail bound of wishart distribution.
\newblock \emph{arXiv preprint arXiv:1212.5860}, 2012.

\end{thebibliography}
\bibliographystyle{icml2023}

\clearpage
\appendix

\section{Additional notes on datasets}\label{sec:app:datasets}

In this appendix, we provide additional analysis justifying the decomposition of robust and spurious \domaindependent features in the real-world datasets.
We also provide details on the construction of \birdcalls.

\subsection{\iwildcam}\label{sec:app:iwildcam}
\tightparagraph{Analysis on \domaindependent features.}
\reffig{iwildcam:dataset_sample} depicts a sample of images from the \iwc training set.
This figure illustrates that animal foregrounds---which are often blurry, occluded, or camouflaged -- are alone insufficient for prediction. 
Extracting habitat features from the background gives useful signal on what species (out of 182 classes) are likely for an image.
We emphasize that $\xcore$ is reliable under realistic distribution shifts for this application: since camera traps monitor wild animals in their natural habitats, adversarial shifts as dramatic as swapping animals between Kenya and Guatemala (\reffig{iwildcam:dataset_sample}) are unlikely.
Further, we show in \refsec{ablations} that being too conservative to this adversarial shift can reduce OOD performance on relevant, widespread shifts (across cameras).

\subsection{\camelyon}\label{sec:app:camelyon}
\tightparagraph{Analysis on \domaindependent features.}
\reffig{camelyon:dataset_sample} depicts a sample of images from the \cam training set.
This figure illustrates that cell morphologies are affected by distributions of patients and their breast cancer stage;
\reffig{camelyon:stages} concretizes how the distribution of cancer stages varies across domains.

We note that unlike \iwildcam and \birdcalls, domains in \camelyon have the same (class-balanced) label distribution.
To understand why models are incentivized to memorize stain color in this task, we plot the class-separated color histograms for the three training domains in \reffig{camelyon_histogram}. We see that, on train, models can learn a threshold function based on the class color means for prediction.

\begin{figure}[h]
    \centering
    \includegraphics[width=\linewidth]{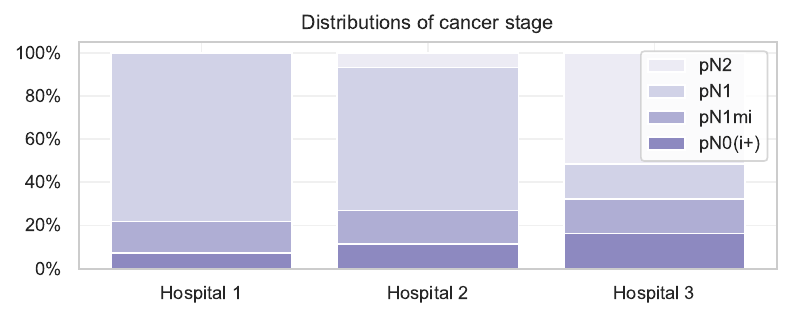}
    \caption{Hospitals vary in the distribution of cancer stages they observe in patients, due to the different patient distributions they service. This in turn affects the causal feature for cancer prediction (cell morphology). \label{fig:camelyon:stages}}
\end{figure}

\begin{figure}[bt]
    \centering
    \includegraphics[width=0.9\linewidth]{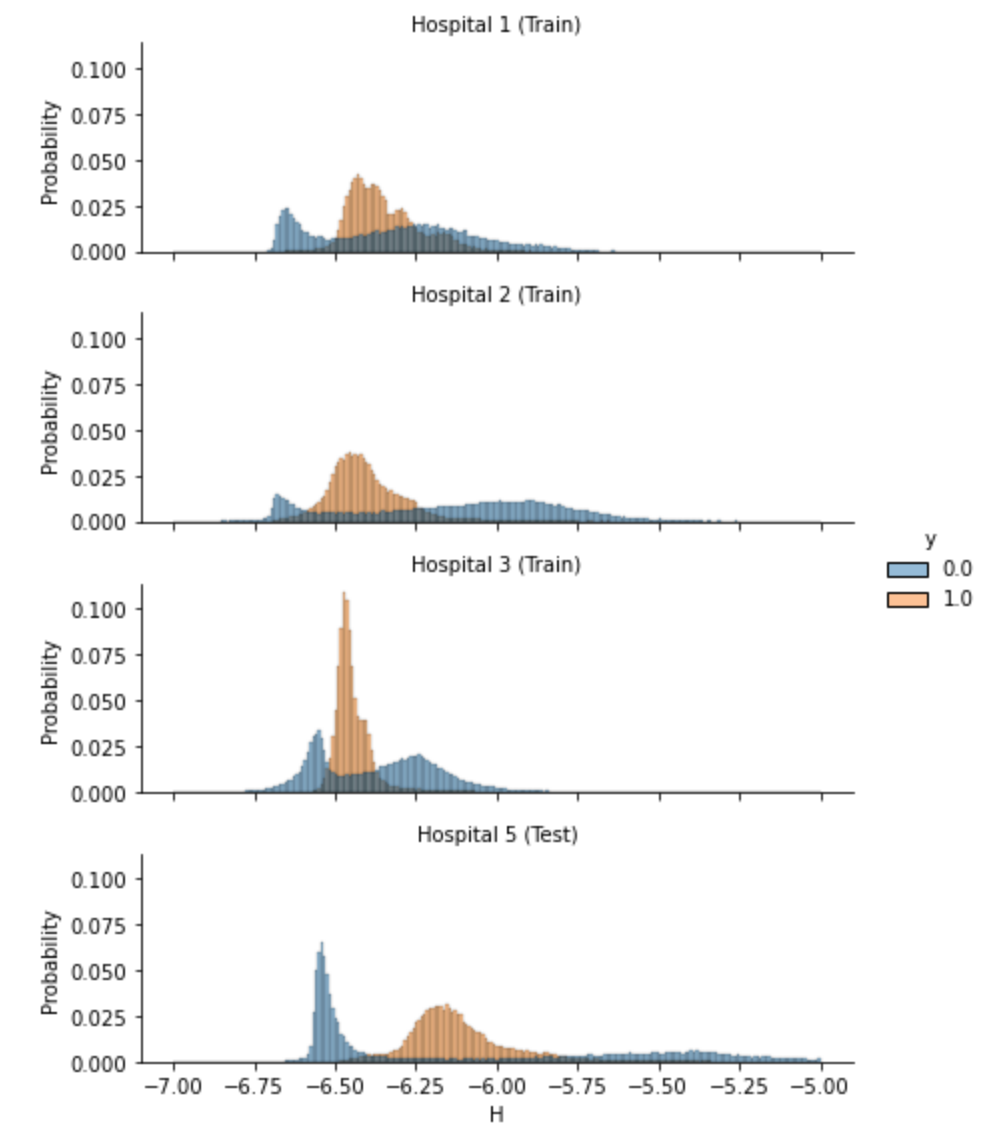}
    \caption{Class-separated color histograms for \camelyon.}
    \label{fig:camelyon_histogram}
\end{figure}

\subsection{\birdcalls}\label{sec:app:birds}
\tightparagraph{Problem setting.} 
To monitor the health of bird populations and their habitats, ornithologists collect petabytes of acoustic recordings from the wild each year.
Machine learning can automate analysis of these recordings by learning to recognize species from audio recordings of their vocalizations.
However, several features vary across the microphones that collect these recordings, such as microphone model, sampling rate, and recording location.
These shifts can degrade model performance on unseen microphones.

\begin{table*}[tb]
\caption{Test-to-test comparison on \birdcalls}
\label{tab:birds_validate_drop}
\begin{tabular}{lllll}
\toprule
 & ID Test Avg Acc & ID Test Macro F1 & OOD Test Avg Acc & OOD Test Macro F1 \\
\midrule
Train on OOD data & 16.7 (0.2) & 4.1 (0.1) & \textbf{84.4 (0.7)} & \textbf{51.9 (0.9)} \\
Train on ID data & \textbf{79.8 (0.4)} & \textbf{70.8 (0.6)} & 44.6 (0.8) & 23.9 (1.0) \\
\bottomrule
\end{tabular}
\end{table*}

\tightparagraph{Dataset construction and statistics.} 
To study targeted augmentations for this setting, we curate a bird recognition dataset by combining publicly released datasets.
\footnote{We release \birdcalls at \href{https://github.com/i-gao/targeted-augs}{this link.}}
The original data is sourced from 32kHz long recordings from \citet{amanda_navine_2022_7078499,w_alexander_hopping_2022_7079124,stefan_kahl_2022_7079380}, which were released alongside expert-annotated time-frequency bounding boxes around observed bird calls.
To build our dataset from these long recordings, we extracted all 5-second chunks in which a single (or no) species makes a call, and then we undersampled majority classes to achieve a more balanced class distribution. 
Our curated dataset, \birdcalls, contains 4,897 audio clips from 12 microphones distributed between the Northeastern United States, Southwest Amazon Basin, and Hawai'i.
Each clip features one of 31 bird species, or no bird (we include an extra class for ``no bird recorded'').
The dataset is split as follows:

\begin{enumerate}
    \item \textbf{Train:} 2,089 clips from 9 microphones
    \item \textbf{ID Validation:} 407 clips from 8 of the 9 microphones in the training set
    \item \textbf{ID Test:} 1,677 clips from the 9 microphones in the training set
    \item \textbf{OOD Test:} 724 clips from 3 different microphones
\end{enumerate}

To train classification models, we convert the 5-second audio clips into Mel spectrograms and train an EfficientNet-B0 on these images, following prior work~\citep{denton2022improving}.
We evaluate ID and OOD performance on their corresponding test sets.
The label distribution of this dataset is shown in \reffig{birds:label_dist}; to account for remaining class imbalance, we report Macro F1 as the evaluation metric.
We show additional samples of the data in \reffig{birds:dataset_sample}.

\tightparagraph{Verifying performance drops.}
We ran checks to verify that observed ID to OOD performance drops were due to distribution shift, and not due to having an innately more difficult OOD Test set.
For these analyses, we further split the OOD Test set into three temporary splits: OOD Train (365 clips), OOD Validation (69 clips), and OOD Test (290).
We then compared the (subsetted) OOD Test performance of models trained on the (ID) Train split + selected on the ID Validation split with models trained on the OOD Train split + selected on the OOD Validation split.
The results are shown in \reftab{birds_validate_drop}.
We see that models perform quite on OOD Test if trained on the same distribution of data (OOD Train).
This verifies that the ID to OOD performance drops are due to distribution shift. 

\tightparagraph{Analysis on \domaindependent features.}
\reffig{birds:dataset_sample} depicts a sample of images from the \birds training set.
This figure shows how habitat features distinctly vary across domains.
Since fine-grained bird species are almost disjoint across regions, habitat features help indicate which species are likely.
Correspondingly, we show in \refsec{ablations} that retaining habitat features improve both ID and OOD performance.

\begin{figure}[tb]
    \centering
    \includegraphics[width=\linewidth]{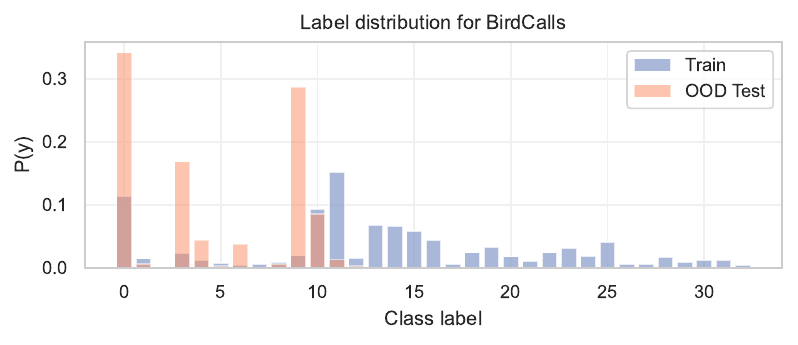}
    \caption{Label distribution of \birdcalls. \label{fig:birds:label_dist}}
\end{figure}

\begin{figure*}[tb]
    \centering
    \includegraphics[width=\linewidth]{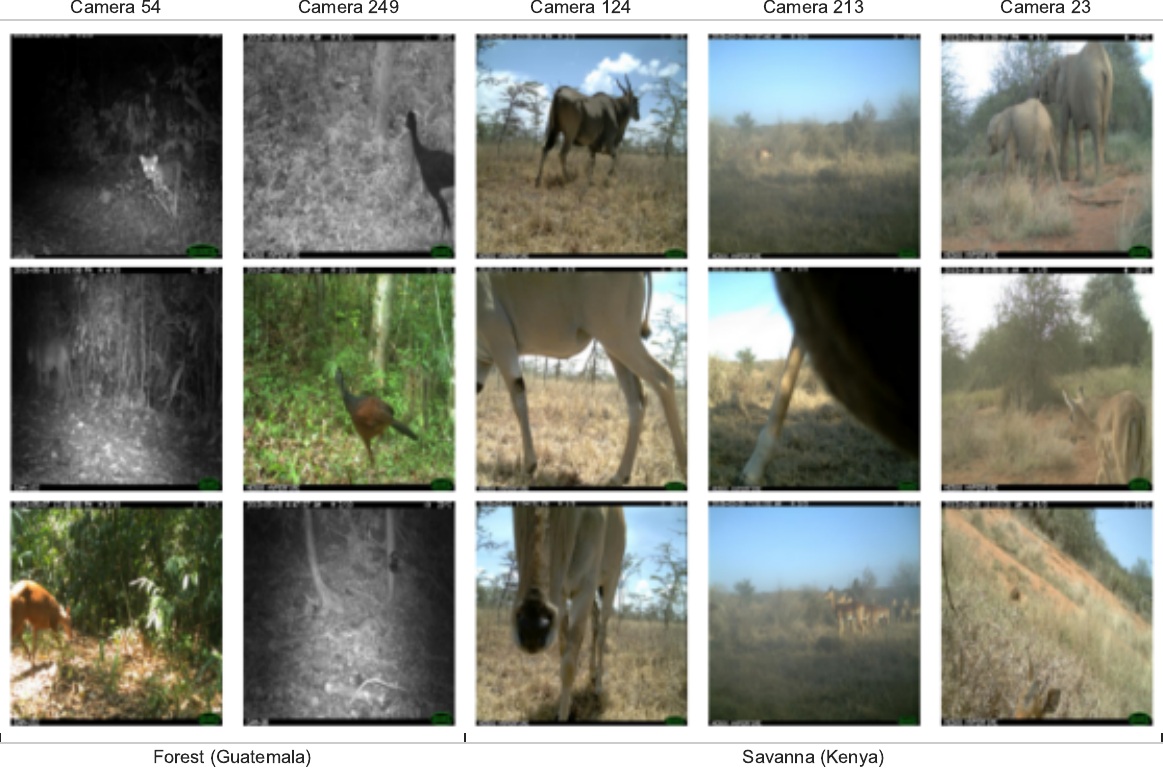}
    \caption{Across domains (columns), both low-level background details $\xspu$ and high-level habitat features $\xcore$ vary. Since $\xcore \not \perp \dom$, domain invariance may eliminate habitat information. In contrast, a targeted augmentation, \cpsamey, randomizes backgrounds between cameras in similar habitats, preserving the ability of the model to use $\xcore$. This is necessary for performance, as foregrounds $\xobj$ can be too camouflaged, distant, blurred, dark, or occluded for even a human annotator's eye. (All images in this figure contain an animal.)\label{fig:iwildcam:dataset_sample}}
\end{figure*}

\begin{figure*}[tb]
    \centering
    \includegraphics[width=\linewidth]{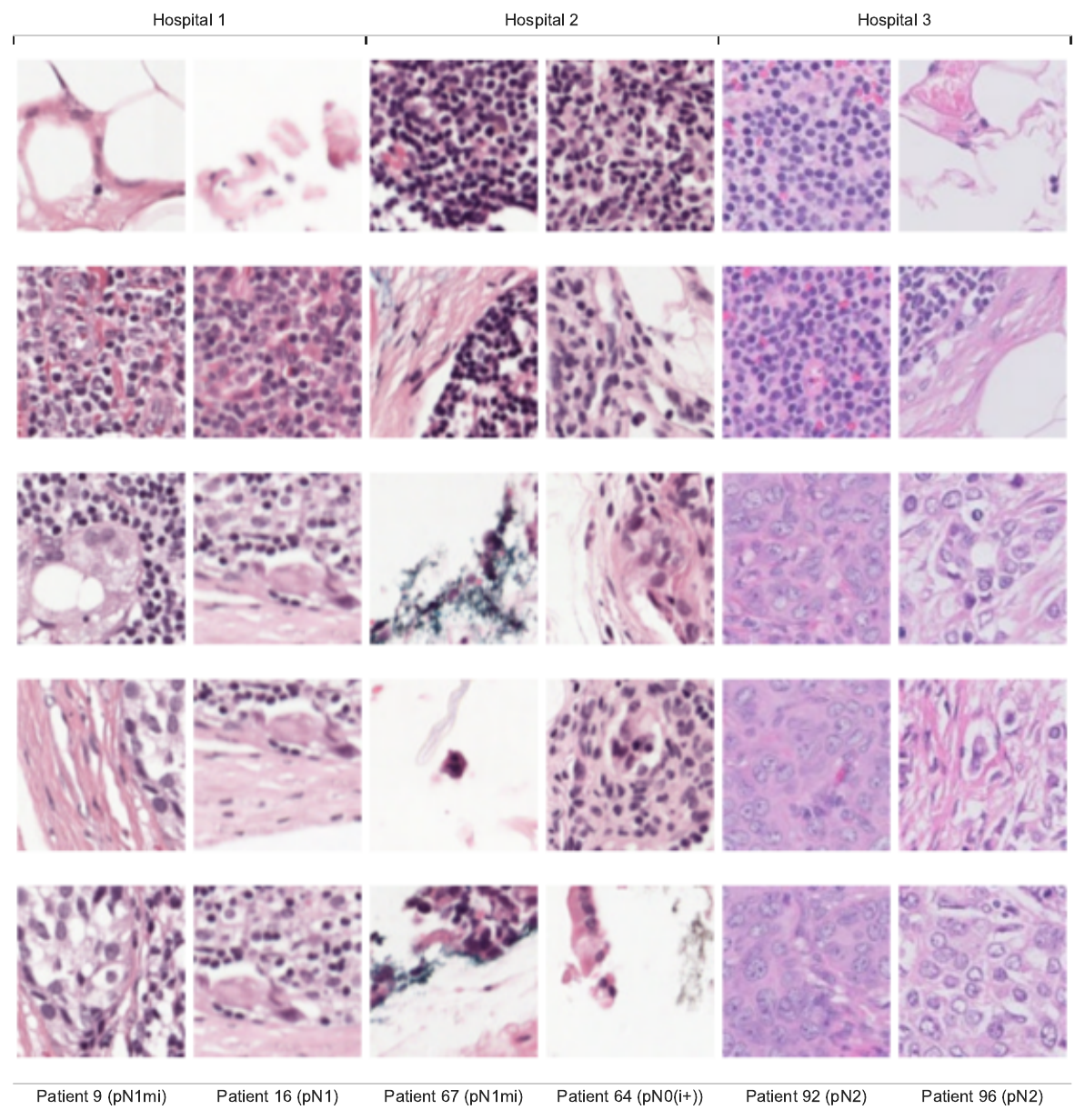}
    \caption{The top two rows depict non-cancerous patches ($y=0$), while the bottom three rows are cancerous patches ($y=1$). Across domains (columns), several features, including distributions of the causal feature (cell morphology), vary. Cell morphology is impacted by the patient distribution of each hospital, as some hospitals have patients with more aggressive cancer staging (\reffig{camelyon:stages}). This leads to different distributions of cell morphologies across domains. While domain invariance would thus eliminate this causal feature, targeted augmentations only randomize features independent of $y$, such as stain color. \label{fig:camelyon:dataset_sample}}
\end{figure*}

\begin{figure*}[tb]
    \centering
    \includegraphics[width=\linewidth]{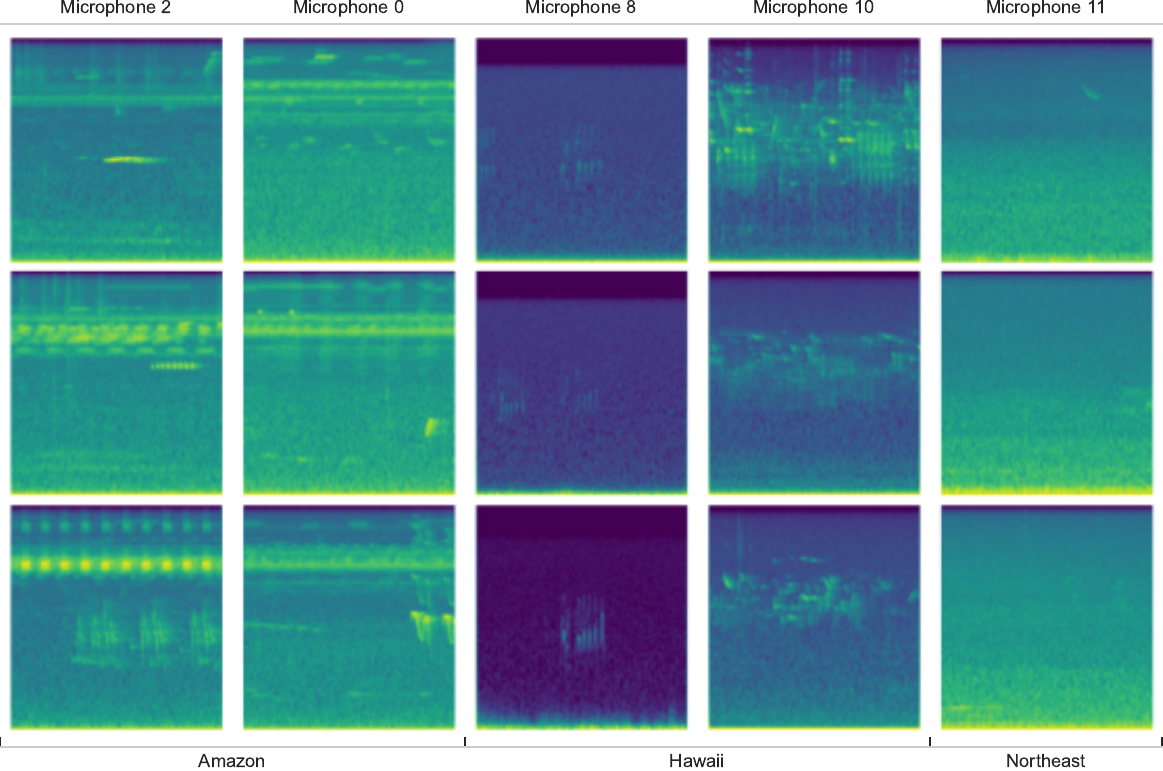}
    \caption{Across domains (columns), recordings vary in their habitat features, such as calls from local insects (left two columns, high frequencies), stronger wind levels (center two columns), or rainfall levels. These habitat features can act as a useful bias for deciding likely labels. Targeted augmentations randomize background noise between microphones located in the same region, preserving this robust feature, while domain invariance eliminates this feature. \label{fig:birds:dataset_sample}}
\end{figure*}

\section{Augmentation details}\label{sec:app:augmentations}
In this appendix, we provide implementation details for the targeted augmentations we study on the real-world datasets.

\subsection{\cpsamey on \iwildcam}
The full \cp protocol is given in Algorithm \ref{alg:cp}. We consider two strategies for selecting the set of valid empty backgrounds $B^\ix$. 
\begin{enumerate}
    \item \textbf{\cpall: all empty train split images.} $B^\ix = \{ (x, y, d) \in \mathcal \Dtrain : y = \text{``empty''}\}$, \ie all augmented examples should have a single distribution of backgrounds. There is a large set of training backgrounds to choose from when executing the procedure -- of $129,809$ training images, $48,021$ are empty images.
    \item \textbf{\cpsamey: empty train split images from cameras that have observed $y^\ix$.} Let $\mathcal Y(d)$ represent the set of labels domain $\dom$ observes. Then $B^\ix = \{ (x, y, d) \in \mathcal \Dtrain : y = \text{``empty''} \text{ and } y^\ix \in \mathcal{Y}(d)\}$.
\end{enumerate}

\begin{algorithm}[h]
  \begin{algorithmic}
    \STATE {\bfseries Input:} {Labeled example $(x^\ix, y^\ix, \dom^\ix)$, binary segmentation mask $m^\ix$, set of images to sample empty images from to use as backgrounds $B^\ix$}
    \IF{$y^\ix$ = ``empty'' or $|B^\ix| = 0$} 
        \STATE {\bfseries Return} $x^\ix$
    \ENDIF
    \STATE Copy out foreground by applying segmentation mask $f^\ix := m^\ix \circ x^\ix$\\
    \STATE Randomly select a background $b \in B^\ix$\\
    \STATE Paste $f^\ix$ onto $b$ and {\bfseries return} $\tilde x^\ix := \text{Paste}(f^\ix, b)$
  \end{algorithmic}
  \caption{\cp}
  \label{alg:cp}
\end{algorithm}

\tightparagraph{Segmentation masks.} The \iwc dataset is curated from real camera trap data collected by the Wildlife Conservation Society and released by \citet{beery2021iwildcam,koh2021wilds}.
\citet{beery2021iwildcam} additionally compute and release segmentation masks for all labeled examples in \iwc. These segmentation masks were extracted by running the dataset through MegaDetector \citep{beery2019megadetector} and then passing regions within detected boxes through an off-the-shelf, class-agnostic detection model, DeepMAC \citep{birodkar2021deepmac}. We use these segmentation masks for our Copy-Paste augmentation.

\begin{table}[tb]
\caption{Pasting onto backgrounds from cameras that have observed the same class during training achieves similar ID and OOD performance to pasting within countries.}
\label{tab:iwildcam_cluster_results}
\resizebox{\linewidth}{!}{
\begin{tabular}{lll}
\toprule
 & ID Test Macro F1 & OOD Test Macro F1 \\
\midrule
Copy-Paste (Same Y) & \textbf{50.2 (0.7)} & 36.5 (0.4) \\
Copy-Paste (Same Country) & 49.3 (0.9) & \textbf{36.7 (0.7)} \\
\bottomrule
\end{tabular}
}
\end{table}

\tightparagraph{Comparison to swapping within countries.}
To confirm that \cpsamey acts to preserve geographic habitat features, we ran an oracle experiment comparing its performance to applying \cp within geographic regions. 
\citet{beery2021iwildcam} released noisy geocoordinates for around half of the locations in \iwildcam. 
Using these coordinates, we inferred the country each camera trap was located in (we merged all cameras of unknown locations into one group, ``unknown country'').
We then applied \cp, pasting animals only onto backgrounds from the same country. 
\reftab{iwildcam_cluster_results} shows that \cpsamey and this oracle have the same performance, suggesting that the Same Y policy indeed preserves geographic habitat features.

\subsection{\stainj on \camelyon}
The full \stainj protocol, originally from  \citet{tellez2018whole}, is given in Algorithm \ref{alg:jitter}. The augmentation uses a pre-specified Optical Density (OD) matrix from \citet{ruifrok2001quantification} to project images from RGB space to a three-channel hematoxylin, eosin, and DAB space before applying a random linear combination.

\begin{algorithm}[h]
  \begin{algorithmic}
    \STATE {\bfseries Input:} {Labeled example $(x^\ix, y^\ix, \dom^\ix)$, normalized OD matrix $M$ \citep{ruifrok2001quantification}, tolerance $\epsilon=1^{-6}$}
    \STATE $S = - \log (x^\ix + \epsilon)M^{-1}$\\
    \STATE Sample $\alpha \sim \text{Uni}(1-\sigma, 1 + \sigma)$\\
    \STATE Sample $\beta \sim \text{Uni}(-\sigma, \sigma)$\\
    \STATE $P = \exp[-(\alpha S + \beta)M] - \epsilon$\\
    \STATE {\bfseries Return} $P$ with each cell clipped to $[0, 255]$
  \end{algorithmic}
  \caption{\stainj Augmentation}
  \label{alg:jitter}
\end{algorithm}

\subsection{\cpsamer on \birdcalls}
After transforming audio clips into mel-spectrograms, we use time-frequency bounding boxes included in the dataset to extract pixels of bird calls.
We then paste these pixels onto spectrograms from the empty (no bird recorded) class, applying Algorithm \ref{alg:cp}.
Finally, we apply color jitter on the spectrograms. 
The goal of jitter is to simulate changes in gain settings across microphones, which affect the coloring of spectrograms.
We consider two strategies for selecting the set of valid empty backgrounds $B^\ix$. 
\begin{enumerate}
    \item \textbf{\cpallr: all empty train split recordings.} $B^\ix = \{ (x, y, d) \in \mathcal \Dtrain : y = \text{``empty''}\}$, \ie all augmented examples should have a single distribution of backgrounds. There is a large set of training backgrounds to choose from when executing the procedure -- of $129,809$ training images, $48,021$ are empty images.
    \item \textbf{\cpsamer: empty train split recordings from microphones in the same region.} Let $R(\dom)$ represent the region (Hawaii, Southwest Amazon Basin, or Northeastern United States) that domain $\dom$ is located in; we provide these annotations in \birdcalls. Then $B^\ix = \{ (x, y, d) \in \mathcal \Dtrain : y = \text{``empty''} \text{ and } R(\dom^\ix) = R(d)\}$.
\end{enumerate}

\onecolumn
\section{Proofs}\label{sec:app:theory}
We present the proofs for results presented in \refsec{theory}.
\subsection{Analyzing \domaindependent features only}\label{app:theory_xinv}
In the proofs, we analyze only the \domaindependent features $\xdom=[\xcore,\xspu]$, disregarding the object features $\xobj$ and noise features $\xnoise$, since the latter two features do not affect our results.
To show this, we first consider the full setting with $x=[\xobj,\xnoise,\xcore,\xspu]$ and compute the model estimate $\thetahat$ by applying the normal equations.
We compute the relevant quantities as 
\begin{align}
\E[xx^\top]=
\begin{pmatrix}
  I & 0 & 0 \\
  0 & I & 0 \\
  0 & 0 & A \\
\end{pmatrix},
\E[yx] = 
\begin{pmatrix}
  \thetastarobj\\
  \thetastarnoise\\
  B\thetastardom\\
\end{pmatrix},
\end{align}
where the blocks correspond to object features $\xobj$, noise features $\xnoise$, and \domaindependent features $[\xcore,\xspu]$ and the matrices $A$ and $B$ depend on the augmentation strategy.
Applying the normal equations yields
  \begin{align}
  \thetahat = 
  \begin{pmatrix}
    \thetastarobj\\
    \thetastarnoise\\
    A^{-1}B\thetastardom.
  \end{pmatrix}
\end{align}
This means that in our infinite-data, finite-domain setting, models perfectly recover $\thetastarobj$ and $\thetastarnoise$ for all augmentation strategies.
Thus, the model incurs zero error from the object and noise dimensions, so these features can also be disregarded in the error computation.

In the rest of the proof, we focus on analyzing the \domaindependent features;
without loss of generality, we assume that the dimensionality of the object and noise features are 0.
In other words, we consider $x=[\xcore,\xspu]$, $\thetastar=\thetastardom=[\thetastarcore,\thetastarspu]$, and $\theta=\thetadom=[\thetacore,\thetaspu]$, all of which are of length $\dimdom$.

\subsection{Models}
\begin{proposition}[Estimator without augmentation]
\label{prop:thetahatnoaug}
Unaugmented training yields the model
\begin{align}
\thetahatnoaug &= (\Sigma + M)^{-1}M\thetastar
\end{align}
where  $M = \frac{1}{D}\sum_{d=1}^D\mud\mudT$ and $\Sigma=\sigma^2I$.
\end{proposition}
\begin{proof}
\begin{align}
\thetahatnoaug 
&= \E[xx^\top]^{-1}\E[xy]\\
&= \left(\frac{1}{D}\sum_{d=1}^D\Sigma + \mud\mudT\right)^{-1}\left(\frac{1}{D}\sum_{d=1}^D\E\left[x(\thetastar\cdot\mud + \varepsilon)\right]\right)\\
&= \left(\Sigma + \frac{1}{D}\sum_{d=1}^D\mud\mudT\right)^{-1}\left(\frac{1}{D}\sum_{d=1}^D\mud\mudT\thetastar\right)\\
&= \left(\Sigma + M\right)^{-1}M\thetastar
\end{align}
\end{proof}

\begin{proposition}[Estimator with \notcrossdomain augmentation]
\label{prop:thetahatstd}
Applying \shelf augmentation yields the model
\begin{align}
\thetahatstd &= (\Sigma + M)^{-1}M\thetastar
\end{align}
where  $M = \frac{1}{D}\sum_{d=1}^D\mud\mudT$ and $\Sigma=\sigma^2I$.
\end{proposition}
\begin{proof}
Applying \shelf augmentations do not change the data distribution over the \domaindependent features.
Thus, $\thetahatstd=\thetahatnoaug$.
Applying \refprop{thetahatnoaug} yields the result.

\end{proof}
\begin{proposition}[Estimator with targeted augmentation]
\label{prop:thetahattgt}
Applying targeted augmentation yields the model
\begin{align}
\thetahattgt =
\begin{pmatrix}
    (\Sigmacore + \Mcore)^{-1}\Mcore\thetastarcore\\
    0
\end{pmatrix}
\end{align}
where  
$\Mcore = \frac{1}{D}\sum_{d=1}^D\mucore\mucoreT$ and $\Sigmacore=\sigma^2I$.
\end{proposition}
\begin{proof}
In the augmented training distribution, input $x$ in domain $d$ is distributed as 
\begin{align}
x \sim N\left(
\begin{pmatrix}
\mucore\\
0
\end{pmatrix},
\Sigmatgt\right),
\end{align}
where $\Sigmatgt = \begin{pmatrix}
    \sigmacore^2I & 0 \\
    0 & (\sigmaspu^2+\tauspu^2)I
\end{pmatrix}$.

Applying the normal equations on the augmented training distribution, we compute $\thetahattgt$ as
\begin{align}
\thetahattgt
&= \E[xx^\top]^{-1}\E[xy]\\
&= \left(\Sigmatgt + \Mtgt\right)^{-1}\Mtgt\thetastar,
\end{align}
where $\Mtgt = \begin{pmatrix}
    \Mcore & 0 \\
    0 & 0
\end{pmatrix}$.

Since we can invert block diagonal matrices block by block, we can compute $\left(\Sigmatgt + \Mtgt\right)^{-1}$ as
\begin{align}
\left(\Sigmatgt + \Mtgt\right)^{-1}
&= 
\begin{pmatrix}
(\sigmacore^2I + \Mcore)^{-1} & 0 \\
0 & \frac{1}{\sigmaspu^2+\tauspu^2}I
\end{pmatrix}.
\end{align}

As a result of the block structure, we can simplify $\thetahattgt$ as
\begin{align}
\thetahattgt =
\begin{pmatrix}
    (\sigmacore^2I + \Mcore)^{-1}\Mcore\thetastarcore\\
    0
\end{pmatrix}
\end{align}
\end{proof}

\begin{proposition}[Estimator with domain-invariant augmentations]
\label{prop:thetahatinv}
Applying domain-invariant augmentation yields the model
\begin{align}
\thetahatinv = 0.
\end{align}
\end{proposition}
\begin{proof}
In the augmented training distribution, input $x$ in domain $d$ is distributed as 
\begin{align}
x \sim N\left(
0,
\Sigma+T\right).
\end{align}
Applying the normal equations thus yields $\thetahatinv=0$.
\end{proof}

\begin{proposition}[Oracle model]
\label{prop:thetahatbest}
Recall that $\thetabest \defeq \argmin_\theta\Rood(\theta)$ is the oracle model that attains optimal performance in the population $\Pmeta$.
The oracle model is
\begin{align}
\thetabest &= (\Sigma + T)^{-1}T\thetastar,
\end{align}
where $\Sigma=\sigma^2I$ and $T=\tau^2I$.
\end{proposition}
\begin{proof}
As the number of domains $\ndom\rightarrow\infty$, $M$ converges to $T$.
Applying the normal equations yields the result.
\end{proof}

\subsection{Computation of ID and OOD errors}
\begin{proposition}[OOD error as a function of $\theta$]
\label{prop:ood-error-theta}
The OOD error of a model $\theta$ is 
\begin{align}
\Rood(\theta)
= \sigmaeps^2 + \theta^\top\Sigma\theta + \left(\thetastar-\theta\right)^\top T\left(\thetastar-\theta\right),
\end{align}
where $\Sigma=\sigma^2I$ and $T=\tau^2I$.
\end{proposition}
\begin{proof}
\begin{align}
R^\mathsf{OOD}(\theta)
=& \E_{x,y,d}\left[\left(y - \theta\cdot x \right)^2\right]\\
=& \E_d\left[\E_{x,y\mid d}\left[\left(y - \theta\cdot x \right)^2\right]\right]\\
=& \E_d\left[\E_{x,y\mid d}\left[\left(\thetastarcore\cdot\mucore + \varepsilon - \theta\cdot x \right)^2\right]\right]\\
=& \sigmaeps^2 + \E_d\left[\left(\thetastar\cdot\mud\right)^2 + \theta^\top\left(\Sigma + \mud\mudT\right)\theta - 2\left(\thetastar\cdot\mud\right)\left(\theta\cdot\mud\right) \right]\\
=& \sigmaeps^2 + \theta^\top\Sigma\theta + \left(\thetastar-\theta\right)^\top\E\left[\mud\mudT\right]\left(\thetastar-\theta\right)\\
=& \sigmaeps^2 + \theta^\top\Sigma\theta + \left(\thetastar-\theta\right)^\top T\left(\thetastar-\theta\right)
\end{align}
\end{proof}

\begin{proposition}[ID error as a function of $\theta$]
\label{prop:id-error-theta}
The ID error of a model $\theta$ is
\begin{align}
\Rid(\theta)
= \sigmaeps^2 + \theta^\top\Sigma\theta + \left(\thetastar-\theta\right)^\top M\left(\thetastar-\theta\right),
\end{align}
where  $M = \frac{1}{D}\sum_{d=1}^D\mud\mudT$ and $\Sigma=\sigma^2I$.
\end{proposition}
\begin{proof}
\begin{align}
\Rid(\theta)
=& \hat{\E}_{x,y,d}\left[\left(y - \theta\cdot x \right)^2\right]\\
=& \hat{\E}_d\left[\E_{x,y\mid d}\left[\left(y - \theta\cdot x \right)^2\right]\right]\\
=& \hat{\E}_d\left[\E_{x,y\mid d}\left[\left(\thetastarcore\cdot\mucore + \varepsilon - \theta\cdot x \right)\right]\right]\\
=& \sigmaeps^2 + \hat{\E}_d\left[\left(\thetastar\cdot\mud\right)^2 + \theta^\top\left(\Sigma + \mud\mudT\right)\theta - 2\left(\thetastar\cdot\mud\right)\left(\theta\cdot\mud\right) \right]\\
=& \sigmaeps^2 + \theta^\top\Sigma\theta + \left(\thetastar-\theta\right)^\top\hat{\E}\left[\mud\mudT\right]\left(\thetastar-\theta\right)\\
=& \sigmaeps^2 + \theta^\top\Sigma\theta + \left(\thetastar-\theta\right)^\top M\left(\thetastar-\theta\right)
\end{align}
\end{proof}

\begin{proposition}[OOD error of the oracle]
\label{prop:best-ood-error}
The OOD error of the oracle model $\thetabest$ is
\begin{align}
\Rood(\thetabest) = \sigmaeps^2 + \frac{\tau^2\sigma^2}{\sigma^2+\tau^2}\|\thetastarcore\|^2.
\end{align}
\end{proposition}
\begin{proof}
  Applying \refprop{thetahatbest} and \refprop{ood-error-theta} yields the following:
\begin{align}
\Rood(\thetabest)
&= \sigmaeps^2 + \thetabest^\top\Sigma\thetabest + (\thetastar-\thetabest)^\top T\left(\thetastar-\thetabest\right)\\
&= \sigmaeps^2 + \frac{\tau^2\sigma^2}{\sigma^2+\tau^2}\|\thetastar\|^2\\
&= \sigmaeps^2 + \frac{\tau^2\sigma^2}{\sigma^2+\tau^2}\|\thetastarcore\|^2.
\end{align}
\end{proof}

\subsection{Proof for \refthm{ood-bound-unaug}}
\label{sec:app:lowerbound}
\begin{theorem}[Excess OOD error without augmentations]
\label{thm:ood-bound-unaug}
If $\ndom<\dimdom$, the expected excess OOD error of the unaugmented model is bounded below as
\begin{align}
\E\left[\Rood(\thetahatnoaug) - \Rood(\thetabest)\right]\ge \frac{\tau^2\varratio \left\|\thetastarcore\right\|^2}{1+\varratio }\left(1-\frac{\ndom}{\dimdom}\right).
\end{align}
\end{theorem}
\begin{proof}
The goal is to lower bound the excess OOD error for the unaugmented estimator $\thetahatnoaug$,
\begin{align}
&\Rood(\thetahatnoaug) - \Rood(\thetabest)\\
=& \sigmaeps^2 + 	\thetahatnoaugT \Sigma\thetahatnoaug + (\thetastar-\thetahatnoaug)^\top T\left(\thetastar-\thetahatnoaug\right) - \Rood(\thetabest) \\
=& \thetastarT M(\Sigma+M)^{-1}\Sigma(\Sigma+M)^{-1}M\thetastar +  \thetastarT \Sigma(\Sigma+M)^{-1}T(\Sigma+M)^{-1}\Sigma\thetastar \\
&- \frac{\tau^2\sigma^2}{\sigma^2+\tau^2}\|\thetastar\|^2.
\end{align}
We first eigendecompose $M$ as
\begin{align}
M = U \diag(\lambda) U^\top.
\end{align}
Using this eigendecomposition, we can compute excess OOD error as
\begin{align}
&\Rood(\thetahatnoaug) - \Rood(\thetabest)\\
=& \thetastarT M(\Sigma+M)^{-1}\Sigma(\Sigma+M)^{-1}M\thetastar +  \thetastarT \Sigma(\Sigma+M)^{-1}T(\Sigma+M)^{-1}\Sigma\thetastar \\
&- \frac{\tau^2\sigma^2}{\sigma^2+\tau^2}\|\thetastar\|^2\\
=& \thetastarT U\diag(v)U^\top\thetastar,
\end{align}
where 
\begin{align}
v_i = \begin{cases}
\frac{\sigma^4(\tau^2-\lambda_i)^2}{(\sigma^2+\tau^2)(\lambda_i + \sigma^2)^2},& i\le D\\
\frac{\tau^4}{(\sigma^2+\tau^2)}, &i > D\\
\end{cases}.
\end{align}
In the above expression, eigenvectors $u_i$ and eigenvalues $\lambda_i$ are random variables, with randomness coming from the draw of domains. We simplify the above expression as
\begin{align}
&\Rood(\thetahatnoaug) - \Rood(\thetabest)\\
=& \thetastarT U\diag(v)U^\top\thetastar\\
  =& \left(\sum_{i=1}^D \frac{\sigma^4(\tau^2-\lambda_i)^2}{(\sigma^2+\tau^2)(\lambda_i + \sigma^2)^2} (u_i^\top\thetastar)^2 + \sum_{i=D+1}^{\dimdom}\frac{\tau^4}{(\sigma^2+\tau^2)}(u_i^\top\thetastar)^2\right).
\end{align}
The first term is always positive, so we can lower bound it by 0, yielding
\begin{align}
&\Rood(\thetahatnoaug) - \Rood(\thetabest)\\
\ge& \sum_{i=D+1}^{\dimdom}\frac{\tau^4}{(\sigma^2+\tau^2)}(u_i^\top\thetastar)^2.
\end{align}
Finally, we compute the expected excess OOD error:
\begin{align}
&\E\left[\Rood(\thetahatnoaug) - \Rood(\thetabest)\right]\\
\ge& \E\left[\sum_{i=D+1}^{\dimdom}\frac{\tau^4}{(\sigma^2+\tau^2)}(u_i^\top\thetastar)^2\right]\\
\ge& \frac{\tau^4}{(\sigma^2+\tau^2)}\sum_{i=D+1}^{\dimdom}\E\left[(u_i^\top\thetastar)^2\right].
\end{align}
We then plug in $\E\left[\left(\theta^\top u_i\right)^2\right]=\|\theta\|^2/\dimdom$ from \reflem{eigenvectors}, which uses the spherical symmetry of $M$'s eigenvectors:
\begin{align}
&\E\left[\Rood\left(\thetahatnoaug\right) - \Rid\left(\thetahatnoaug\right)\right]\\
\ge& \frac{\tau^4}{(\sigma^2+\tau^2)}\sum_{i=D+1}^{\dimdom}\E\left[(u_i^\top\thetastar)^2\right]\\
=& \frac{\tau^4}{(\sigma^2+\tau^2)}\frac{\dimdom-D}{\dimdom}\left\|\thetastar\right\|^2\\
\ge& \frac{\tau^2\varratio \left\|\thetastar\right\|^2}{1+\varratio }\cdot\frac{\dimdom-D}{\dimdom}\\
=& \frac{\tau^2\varratio \left\|\thetastarcore\right\|^2}{1+\varratio }\cdot\frac{\dimdom-D}{\dimdom}.
\end{align}
where $\gamma=\tau/\sigma$.
\end{proof}

\begin{lemma}
\label{lem:eigenvectors}
Let $\theta\in\BR^{m}$ be a fixed vector, and let $u_i$ be eigenvectors with the $i$th largest eigenvalue for a random matrix $A = \frac{1}{k}\sum_{d=1}^k\zd\zd^\top$, where $\zd$ is drawn from an isotropic Gaussian as $\zd\sim N(0,s^2I_m)$. For all $i=1,\ldots,m$,
\begin{align}
\E\left[(\theta^\top u_i)^2\right] = \E[(\theta^\top u_i)^2\mid\lambda_1,\ldots,\lambda_m] = \frac{\|\theta\|^2}{m}
\end{align}
\end{lemma}
\begin{proof}
Since $\zd$ is sampled from an isotropic Gaussian, $A$'s unit eigenvectors are uniformly distributed on the unit sphere.
Thus, we can simplify the expectation as follows:
\begin{align}
\E\left[(\theta^\top u_i)^2\right] 
&=\theta^\top\E\left[ u_iu_i^\top\right]\theta\\
&=\theta^\top\left(\frac{1}{m}I\right)\theta\\
&= \frac{\|\theta\|^2}{m}
\end{align}
By the same symmetry argument, we get the same expected value even when conditioned on the eigenvalues,
\begin{align}
\E\left[(\theta^\top u_i)^2\mid\lambda_1,\ldots,\lambda_m\right] 
&= \frac{\|\theta\|^2}{m}.
\end{align}
\end{proof}

\begin{restatable}[Excess OOD error with \notcrossdomain augmentations]{cor}{oodboundstd}
\label{cor:ood-bound-std}
If $\ndom<\dimdom$, the expected excess OOD error of the \notcrossdomain model is bounded below as
\begin{align*}
\E\left[\Rood(\thetahatstd) - \Rood(\thetabest)\right]\ge \frac{\tau^2\varratio \left\|\thetastarcore\right\|^2}{1+\varratio }\left(1-\frac{\ndom}{\dimdom}\right).
\end{align*}
\end{restatable}
\begin{proof}
  This follows from \refthm{ood-bound-unaug} and \refprop{thetahatstd}.
\end{proof}

\subsection{Proof for \refthm{ood-bound-tgt-zhu}}\label{sec:app:proof-ood-bound-tgt-zhu}
We first present \refthm{ood-bound-tgt-zhu} and its proof, including a more general theorem statement before it was simplified for the main text.
\begin{theorem}[Excess OOD error with targeted augmentations]
\label{thm:ood-bound-tgt-zhu-general}
Assume $\varratio >1$. For any $0<r_0\le1$ and large enough $\ndom$ such that $\ndom>2(\dimcore+2)\log(4\ndom\dimcore/r_0)$, the excess OOD error is bounded as
\begin{align}
&\E\left[\Rood(\thetahattgt) - \Rood(\thetabest)\right]
\le \frac{\tau^2\varratio \|\thetastarcore\|^2 }{1+\varratio }\left(\frac{r_0}{D} + \frac{2(\dimcore+2)\log(4D\dimcore/r_0)}{D\left(1+\varratio \left(1-\sqrt{\frac{2(\dimcore+2)\log(4D\dimcore/r_0)}{D}}\right)\right)^2}\right).
\end{align}
Furthermore, for any $0<r<1$ and large enough $\ndom$ such that $D>2(\dimcore+2)\log(4D\dimcore)/(1-r)^2$,
\begin{align}
\E\left[\Rood(\thetahattgt) - \Rood(\thetabest)\right]
&\le \frac{\tau^2\varratio \left\|\thetastarcore\right\|^2}{1+\varratio }\left(\frac{1}{\ndom} + \frac{2\log(4\ndom\dimcore)(\dimcore+2)}{\ndom(1+\varratio r)^2}\right).
\end{align}
\end{theorem}

\begin{proof}
Applying \refprop{ood-bound-tgt} and \reflem{min-eigval-zhu} yields
\begin{align}
&\E\left[\Rood(\thetahattgt) - \Rood(\thetabest)\right]\\
\le&
 \frac{\tau^2\varratio }{1+\varratio }\|\thetastarcore\|^2 \left(\frac{\eta^2}{(1+\varratio (1-\eta))^2}+ \delta\right)\\
=&\frac{\tau^2\varratio }{1+\varratio }\|\thetastarcore\|^2 \left(\delta + \frac{2(\dimcore+2)\log(4\dimcore/\delta)}{D\left(1+\varratio \left(1-\sqrt{\frac{2(\dimcore+2)\log(4\dimcore/\delta)}{D}}\right)\right)^2}\right)\\
=&\frac{1}{D}\frac{\tau^2\varratio }{1+\varratio }\|\thetastarcore\|^2 \left(\delta D + \frac{2(\dimcore+2)\log(4\dimcore/\delta)}{\left(1+\varratio \left(1-\sqrt{\frac{2(\dimcore+2)\log(4\dimcore/\delta)}{D}}\right)\right)^2}\right).
\end{align}
We will discuss the assumptions needed to apply \refprop{ood-bound-tgt} and \reflem{min-eigval-zhu} in a subsequent paragraph. Before we do that, we will pick $\delta$ as $\delta=r_0/D$ for any constant $0<r_0\le1$, in which case $0<\delta<1$ for $D>1$. Then, we can simplify the expression as
\begin{align}
&\E\left[\Rood(\thetahattgt) - \Rood(\thetabest)\right]\\
&\le\frac{1}{D}\frac{\tau^2\varratio }{1+\varratio }\|\thetastarcore\|^2 \left(\delta D + \frac{2(\dimcore+2)\log(4\dimcore/\delta)}{\left(1+\varratio \left(1-\sqrt{\frac{2(\dimcore+2)\log(4\dimcore/\delta)}{D}}\right)\right)^2}\right)\\
&\le\frac{\tau^2\varratio \|\thetastarcore\|^2 }{D(1+\varratio )}\left(r_0 + \frac{2(\dimcore+2)\log(4D\dimcore/r_0)}{\left(1+\varratio \left(1-\sqrt{\frac{2(\dimcore+2)\log(4D\dimcore/r_0)}{D}}\right)\right)^2}\right).
\end{align}

In order to apply \refprop{ood-bound-tgt} and \reflem{min-eigval-zhu} above, we need to satisfy the following assumptions:
\begin{itemize}
\item $\eta < 1$
\item $\sigma^2 < \tau^2$,
\end{itemize}
where $\eta = \sqrt{\frac{2(\dimcore+2)\log(4D\dimcore/r_0)}{D}}$ in this case.
The first assumption is equivalent to 
\begin{align}
D&>2(\dimcore+2)\log(4D\dimcore/r_0).
\end{align}
This concludes the proof of the general statement.

Now, we will simplify the expression for clarity.
First, let's set $r_0=1$. This yields:
\begin{align}
&\E\left[\Rood(\thetahattgt) - \Rood(\thetabest)\right]\\
\le& \frac{\tau^2\varratio \|\thetastarcore\|^2 }{D(1+\varratio )}\left(1 + \frac{2(\dimcore+2)\log(4D\dimcore)}{\left(1+\varratio \left(1-\sqrt{\frac{2(\dimcore+2)\log(4D\dimcore)}{D}}\right)\right)^2}\right).
\end{align}
Now, we will bound 
\begin{align}
1-\sqrt{\frac{2(\dimcore+2)\log(4\ndom\dimcore)}{D}} > r
\end{align}
for any $0<r<1$. To do so, we further assume large enough $\ndom$ such that $\ndom>2(\dimcore+2)\log(4\ndom\dimcore)/(1-r)^2$.
Then, we can simplify the bound as 
\begin{align}
&\E\left[\Rood(\thetahattgt) - \Rood(\thetabest)\right]\\
\le& \frac{\tau^2\varratio \|\thetastarcore\|^2 }{D(1+\varratio )}\left(1 + \frac{2(\dimcore+2)\log(4D\dimcore)}{(1+\varratio r)^2}\right).
\end{align}
\end{proof}

\begin{proposition}
\label{prop:ood-bound-tgt}
Let $\lambdamin,\lambdamax$ be the minimum and maximum eigenvalue of $\Mcore$, respectively. If $\sigma<\tau$ and $\tau^2(1-\eta)\le\lambdamin\le\lambdamax\le\tau^2(1+\eta+\eta^2)$ with probability greater than $1-\delta$ and $\eta<1$, then
\begin{align}
\E\left[\Rood(\thetahattgt) - \Rood(\thetabest)\right]
\le \frac{\tau^2\varratio }{1+\varratio }\|\thetastarcore\|^2 \left(\frac{\eta^2}{(1+\varratio (1-\eta))^2}+ \delta\right)
\end{align}
\end{proposition}
\begin{proof}
The excess OOD error of $\thetahattgt$ is 
\begin{align}
&\Rood(\thetahattgt) - \Rood(\thetabest)\\
=& \sigmaeps^2 + 	\thetahattgtT \Sigma\thetahattgt + (\thetastar-\thetahattgt)T\left(\thetastar-\thetahattgt\right) - \Rood(\thetabest) \\
=& \sigmaeps^2 + 	\thetahattgtcoreT \Sigmacore\thetahattgtcore + (\thetastarcore-\thetahattgtcore)^\top\Tcore\left(\thetastarcore-\thetahattgtcore\right) - \Rood(\thetabest)\\
=& 	\thetahattgtcoreT \Sigmacore\thetahattgtcore + (\thetastarcore-\thetahattgtcore)^\top\Tcore\left(\thetastarcore-\thetahattgtcore\right) - \frac{\tau^2\sigma^2}{\sigma^2+\tau^2}\|\thetastarcore\|^2\\
=& 	\thetastarcoreT \Mcore(\Sigmacore+\Mcore)^{-1}\Sigmacore(\Sigmacore+\Mcore)^{-1}\Mcore\thetastarcore\\
&+ 	\thetastarcoreT \Sigmacore(\Sigmacore+\Mcore)^{-1}\Tcore(\Sigmacore+\Mcore)^{-1}\Sigmacore\thetastarcore - \frac{\tau^2\sigma^2}{\sigma^2+\tau^2}\|\thetastarcore\|^2.
\end{align}
We first eigendecompose $\Mcore$ as
\begin{align}
\Mcore = U \diag(\lambda) U^\top.
\end{align}
Using this eigendecomposition, we can compute excess OOD error as
\begin{align}
&\Rood(\thetahattgt) - \Rood(\thetabest)\\
=& 	\thetastarcoreT \Mcore(\Sigmacore+\Mcore)^{-1}\Sigmacore(\Sigmacore+\Mcore)^{-1}\Mcore\thetastarcore\\
&+ 	\thetastarcoreT \Sigmacore(\Sigmacore+\Mcore)^{-1}\Tcore(\Sigmacore+\Mcore)^{-1}\Sigmacore\thetastarcore - \frac{\tau^2\sigma^2}{\sigma^2+\tau^2}\|\thetastarcore\|^2\\
=& 	\thetastarcoreT U\diag(v)U^\top\thetastarcore
\end{align}
where 
\begin{align}
v_i 
&= \frac{\sigma^2\lambda_i^2+\sigma^4\tau^2}{(\lambda_i+\sigma^2)^2} - \frac{\tau^2\sigma^2}{\sigma^2+\tau^2}\\
&= \frac{\sigma^4(\tau^2-\lambda_i)^2}{(\sigma^2+\tau^2)(\lambda_i + \sigma^2)^2}.
\end{align}
We can rewrite the excess OOD error as
\begin{align}
&\Rood(\thetahattgt) - \Rood(\thetabest)\\
=& 	\thetastarcoreT U\diag(v)U^\top\thetastarcore\\
=& \sum_{i=1}^{\dimcore} v_i (	\thetastarcoreT u_i)^2\\
=& \sum_{i=1}^{\dimcore} \frac{\sigma^4(\tau^2-\lambda_i)^2}{(\sigma^2+\tau^2)(\lambda_i + \sigma^2)^2} (	\thetastarcoreT u_i)^2.
\end{align}
We now bound the excess OOD error by applying the bound on $\lambdamin$ and $\lambdamax$. Recall that we assume  $\tau^2(1-\eta)\le\lambdamin\le\lambdamax\le\tau^2(1+\eta+\eta^2)$ with probability greater than $1-\delta$. Applying \reflem{upper-bound-v}, if $\tau^2(1-\eta)\le\lambdamin\le\lambdamax\le\tau^2(1+\eta+\eta^2)$ and $\eta < 1$, then the following holds:
\begin{align}
&\Rood(\thetahattgt) - \Rood(\thetabest)\\
=& \sum_{i=1}^{\dimcore} \frac{\sigma^4(\tau^2-\lambda_i)^2}{(\sigma^2+\tau^2)(\lambda_i + \sigma^2)^2} (	\thetastarcoreT u_i)^2\\
\le& \frac{\sigma^4\tau^4\eta^2}{(\sigma^2+\tau^2)(\tau^2(1-\eta) + \sigma^2)^2}\|\thetastarcore\|^2\\
=& \frac{\tau^2\varratio \eta^2}{(1+\varratio )(1+\varratio (1-\eta))^2}\|\thetastarcore\|^2.
\end{align}

We now bound the expected value of the excess OOD error. Because the above bound holds with probability greater than $1-\delta$ (because the eigenvalue bounds hold with probability greater than $1-\delta$), we first obtain the expected value by applying the total law of expectation:
\begin{align}
&\E\left[\Rood\left(\thetahattgt\right) - \Rood\left(\thetabest\right)\right]\\
\le&(1-\delta)\E\left[\Rood\left(\thetahattgt\right) - \Rood\left(\thetabest\right)\midd\tau^2(1-\eta)\le\lambdamin\le\lambdamax\le\tau^2(1+\eta+\eta^2)\right] \\
&+ \delta\E\left[\Rood\left(\thetahattgt\right) - \Rood\left(\thetabest\right)\midd\lambdamin<\tau^2(1-\eta)\text{ or }\lambdamax>\tau^2(1+\eta+\eta^2)\right]\\
\le & \frac{\tau^2\varratio \eta^2}{(1+\varratio )(1+\varratio (1-\eta))^2}\|\thetastarcore\|^2\\
&+ \delta \E\left[\sum_{i=1}^{\dimcore} \frac{\sigma^4(\tau^2-\lambda_i)^2}{(\sigma^2+\tau^2)(\lambda_i + \sigma^2)^2} (	\thetastarcoreT u_i)^2\midd\lambdamin<\tau^2(1-\eta)\text{ or }\lambdamax>\tau^2(1+\eta+\eta^2)\right]\\
\le & \frac{\tau^2\varratio \eta^2}{(1+\varratio )(1+\varratio (1-\eta))^2}\|\thetastarcore\|^2 + \delta \frac{\sigma^4\tau^4}{(\sigma^2+\tau^2)\sigma^4} \|\thetastarcore\|^2\\
= & \frac{\tau^2\varratio }{1+\varratio }\|\thetastarcore\|^2 \left(\frac{\eta^2}{(1+\varratio (1-\eta))^2}+ \delta\right).
\end{align}
  In the second to last step, we upper bound the second term by the maximum value for $\lambda_i\in[0,\infty)$, using the fact that $\lambda_i\ge0$ as $\Mcore$ is positive semidefinite. From \reflem{derivative-v}, the upper bound is the higher of the value at $\lambda_i=0$, which is $\frac{\sigma^4\tau^4}{(\sigma^2+\tau^2)\sigma^4}\|\thetastarcore\|^2$, and $\lim_{\lambda_i\to\infty}\frac{\sigma^4(\tau^2-\lambda_i)^2}{(\sigma^2+\tau^2)(\lambda_i + \sigma^2)^2}\|\thetastarcore\|^2=\frac{\sigma^4}{\sigma^2+\tau^2}\|\thetastarcore\|^2$. Because $\varratio >1$, the former is higher, \ie a more conservative upper bound.
\end{proof}

\begin{lemma}
\label{lem:derivative-v}
Let $f(z) = \frac{(\tau^2-z)^2}{(\sigma^2+z)^2}$. The derivative of $f$ is
\begin{align}
\frac{d}{dz}f(z) = -\frac{2(\tau^2-z)(\sigma^2+\tau^2)}{(\sigma^2+z)^3},
\end{align}
and $f$ is decreasing in $(-\sigma^2,\tau^2)$ and increasing in $(\tau,\infty)$.
\end{lemma}
\begin{proof}
Taking the derivative, we get
\begin{align}
\frac{d}{dz}f(z) = -\frac{2(\tau^2-z)(\sigma^2+\tau^2)}{(\sigma^2+z)^3},
\end{align}
\end{proof}

\begin{lemma}
\label{lem:upper-bound-v}
For $z,\eta,\sigma,\tau$ such that $\tau^2(1-\eta)\le z \le \tau^2(1+\eta+\eta^2)$, $\sigma<\tau$, and $0\le\eta<1+\sigma^2/\tau^2$, 
\begin{align}
\frac{(\tau^2-z)^2}{(\sigma^2+z)^2} \le \frac{\tau^4\eta^2}{(\sigma^2+\tau^2(1-\eta))^2}.
\end{align}
\end{lemma}
\begin{proof}
Let $f(z) = \frac{(\tau^2-z)^2}{(\sigma^2+z)^2}$.
Because $f(z)$ is decreasing for $-\sigma^2 < z < \tau^2$ and increasing for $z>\tau^2$ (\reflem{derivative-v}), we can bound $f(z)$ for $\tau^2(1-\eta)\le z \le \tau^2(1+\eta+\eta^2)$ as
\begin{align}
f(z)\le \max\left(
\frac{\tau^4\eta^2}{(\sigma^2+\tau^2(1-\eta))^2},
\frac{\tau^4(\eta+\eta^2)^2}{(\sigma^2+\tau^2(1+\eta+\eta^2))^2}
\right),
\end{align}
if $\eta<1+1/\varratio $.
We now show that 
\begin{align}
\frac{\tau^4\eta^2}{(\sigma^2+\tau^2(1-\eta))^2} > \frac{\tau^4(\eta+\eta^2)^2}{(\sigma^2+\tau^2(1+\eta+\eta^2))^2}
\end{align}
for $\eta>0$. We can simplify the difference between these two quantities as 
\begin{align}
&\frac{\tau^4\eta^2}{(\sigma^2+\tau^2(1-\eta))^2} - \frac{\tau^4(\eta+\eta^2)^2}{(\sigma^2+\tau^2(1+\eta+\eta^2))^2} \\
=&\frac{\eta^3(\eta+2)(\sigma^2+\tau^2)(-\sigma^2+2\tau^2\eta+\tau^2)}{(\sigma^2-\tau^2\eta+\tau^2)^2(\sigma^2+\tau^2\eta^2+\tau^2\eta+\tau^2)^2}.
\end{align}
The above is positive if $-\sigma^2+2\tau^2\eta+\tau^2>0$, which will be the case for $\eta>0$ and $\tau^2>\sigma^2$.
\end{proof}

\begin{lemma}
\label{lem:min-eigval-zhu}
Let $\lambdamin,\lambdamax$ be the minimum and maximum eigenvalues of $\Mcore$, respectively. With probability greater than $1-\delta$, the eigenvalues can be bounded as
\begin{align}
\lambdamin &\ge \tau^2\left(1-\sqrt{\frac{-2(\dimcore+2)\log(\delta/4\dimcore)}{D}}\right)\\
\lambdamax &\le \tau^2\left(1+\sqrt{\frac{-2(\dimcore+2)\log(\delta/4\dimcore)}{D}}+\frac{-2(\dimcore+2)\log(\delta/4\dimcore)}{D}\right) 
\end{align}
\end{lemma}
\begin{proof}
  We apply equations 1 and 6 from \citet{zhu2012short} and the union bound. Note that the bounds can be written as 
\begin{align}
\tau^2(1-\eta) \le \lambdamin \le \lambdamax \le \tau^2(1+\eta+\eta^2),
\end{align}
where $\eta=\sqrt{\frac{-2(\dimcore+2)\log(\delta/4\dimcore)}{D}}$.
\end{proof}

\subsection{Proof for \refthm{gap-bound}}
\gapbound*
\begin{proof}
First, we simplify the upper bound further, by picking $r=1/\varratio $ and by bounding $\ndom<\dimdom$:
\begin{align}
  &\E\left[\Rood(\thetahattgt) - \Rood(\thetabest)\right]\\
  &\le \frac{\tau^2\varratio \left\|\thetastarcore\right\|^2}{1+\varratio }\left(\frac{1}{\ndom} + \frac{2\log(4\ndom\dimcore)(\dimcore+2)}{\ndom(1+\varratio r)^2}\right)\\
  &\le \frac{\tau^2\varratio \left\|\thetastarcore\right\|^2}{1+\varratio }\left(\frac{1}{\ndom} + \frac{2\log(4\ndom\dimcore)(\dimcore+2)}{4\ndom}\right)\\
  &\le \frac{\tau^2\varratio \left\|\thetastarcore\right\|^2}{1+\varratio }\left(\frac{2+\log(4\dimdom\dimcore)(\dimcore+2)}{2\ndom}\right).
\end{align}

Now, we compare with the lower bound. The gap is:
\begin{align}
&\E\left[\Rood(\thetahatstd)\right] - \E\left[\Rood(\thetahattgt)\right]\\
&=\E\left[\Rood(\thetahatstd) - \Rood(\thetabest)\right] - \E\left[\Rood(\thetahattgt) - \Rood(\thetabest)\right]\\
&\ge \frac{\tau^2\varratio \left\|\thetastarcore\right\|^2}{1+\varratio }\left(1-\frac{\ndom}{\dimdom} - \frac{2+\log(4\dimdom\dimcore)(\dimcore+2)}{2\ndom}\right)
\end{align}

We apply Lemma \ref{lem:gap-polynomial}, noting that $1 < \log(2\dimdom)(\dimcore+2)$ if $\dimdom \ge 2$, \ie as long as we have at least one robust \domaindependent feature and one spurious \domaindependent feature.
\begin{align}
&\E\left[\Rood(\thetahatstd)\right] - \E\left[\Rood(\thetahattgt)\right]\\
&\ge \frac{\tau^2\varratio \left\|\thetastarcore\right\|^2}{1+\varratio }\left( -\left(\ndom-\frac{\dimdom}{2}\right)^2 + \frac{\dimdom^2}{4} - 2\dimdom(\dimcore+2)\log(2\dimdom) \right)\label{eq:gap-lb}
\end{align}

  We now find the conditions where the gap (Equation~\ref{eq:gap-lb}) is positive:
\begin{align}
&-\left(\ndom-\frac{\dimdom}{2}\right)^2 + \frac{\dimdom^2}{4} - 2\dimdom(\dimcore+2)\log(2\dimdom) > 0\\
\iff & \left(\ndom-\frac{\dimdom}{2}\right)^2 < \frac{\dimdom^2}{4}-2\dimdom(\dimcore+2)\log(2\dimdom)\\
\iff & \frac{\dimdom}{2} - \sqrt{\frac{\dimdom^2}{4}-2\dimdom(\dimcore+2)\log(2\dimdom)} < \ndom < \frac{\dimdom}{2} + \sqrt{\frac{\dimdom^2}{4}-2\dimdom(\dimcore+2)\log(2\dimdom)}\\
\impliedby & 4(\dimcore+2)\log(2\dimdom) < \ndom < \dimdom - 4(\dimcore+2)\log(2\dimdom),
\end{align}
where the last step applies $\sqrt{x-y} > \sqrt{x} - \sqrt{y}$ for $0<y<x$. 
For the above computation to go through, we need to ensure that the term in the square root is positive:
\begin{align}
\frac{\dimdom^2}{4}-2\dimdom(\dimcore+2)\log(2\dimdom) > 0.
\end{align}
With algebra, we can show that this is equivalent to
\begin{align}
\dimcore < \frac{\dimdom}{8\log(2\dimdom)}-2.
\end{align}
In addition, we need to satisfy the assumption for \refthm{ood-bound-tgt-zhu}:
\begin{align}
\ndom>2(\dimcore+2)\log(4\ndom\dimcore)/(1-1/\varratio )^2,
\end{align}
which would be implied by 
\begin{align}
\ndom>4(\dimcore+2)\log(2\dimdom)/(1-1/\varratio )^2
\end{align}
for $D<\dimdom$.
We compare this above minimum value on $\ndom$ with the minimum value of $\ndom$ for which there is a gap, we see that $4(\dimcore+2)\log(2\dimdom)/(1-1/\varratio )^2$ is larger by a factor of $(1-1/\varratio )^{-2}$.
Thus, we can show a gap when
\begin{align}
4(\dimcore+2)\log(2\dimdom)/(1-1/\varratio )^2 < \ndom < \dimdom - 4(\dimcore+2)\log(2\dimdom).
\end{align}

Finally, we want to show that the above is a non-empty range, with
\begin{align}
&\frac{4(\dimcore+2)\log(2\dimdom)}{(1-1/\varratio )^2} < \dimdom - 4(\dimcore+2)\log(2\dimdom)\\
\iff & \dimcore< \frac{\dimdom}{4\log(2\dimdom)(1+(1-1/\varratio )^{-2})} -2.
\end{align}
Comparing with the earlier condition on $\dimcore$, we see that this is a stronger condition.

Because $\thetahatnoaug=\thetahatstd$, the same result applies in comparison to $\thetahatstd$ as well.
\end{proof}

\begin{lemma}[Negative polynomial lower bound for gap term.]
\label{lem:gap-polynomial}
If $1 < \log(2\dimdom)(\dimcore+2)$ and $\ndom\dimdom > 1$, 
\begin{align}
1-\frac{\ndom}{\dimdom} - \frac{2+\log(4\dimdom\dimcore)(\dimcore+2)}{2\ndom}
> 
  -\left(\ndom-\frac{\dimdom}{2}\right)^2 + \frac{\dimdom^2}{4} - 2\dimdom(\dimcore+2)\log(2\dimdom)
\end{align}
\end{lemma}

\begin{proof}
Since $1 < \log(2\dimdom)(\dimcore +2)$,
\begin{align}
&\frac1D + \frac{\log(2\dimdom)(\dimcore+2)}{\ndom} < \frac{2 \log(2\dimdom)(\dimcore+2)}{\ndom} \\
\implies&\left(1-\frac{\dimdom}{\ndom}\right) + \frac{\dimdom}{\ndom^2} + \frac{\dimdom}{\ndom}\left(\frac{\log(2\dimdom)(\dimcore +2)}{\ndom}\right) < \left(1-\frac{\dimdom}{\ndom}\right) + \frac{2\dimdom}{\ndom} \left( \frac{\log(2\dimdom)(\dimcore+2)}{\ndom}\right)\\
\implies&\left(1-\frac{\dimdom}{\ndom}\right) + \frac{\dimdom+\dimdom\log(2\dimdom)(\dimcore+2)}{\ndom^2}<\left(1-\frac{\dimdom}{\ndom}\right)+\frac{2\dimdom(\dimcore+2)\log(2\dimdom)}{\ndom^2}\\
\implies&\left(1-\frac{\dimdom}{\ndom}\right) + \frac{\dimdom+\dimdom\log(2\dimdom)(\dimcore+2)}{\ndom^2}<\ndom\dimdom\left(1-\frac{\dimdom}{\ndom}\right)+\frac{2\dimdom^2(\dimcore+2)\log(2\dimdom)}{\ndom}
\end{align}
Since $\dimcore \le \dimdom$, we know that $\frac12 \log(4\dimdom\dimcore) \le \log(2\dimdom)$. 
\begin{align}
\implies&\left(1-\frac{\dimdom}{\ndom}\right) + \frac{2\dimdom+\dimdom\log(4\dimdom\dimcore)(\dimcore+2)}{2D^2}<\ndom\dimdom\left(1-\frac{\dimdom}{\ndom}\right)+\frac{2\dimdom^2(\dimcore+2)\log(2\dimdom)}{\ndom}\\
\implies&\frac{\ndom}{\dimdom}\left(1-\frac{\dimdom}{\ndom}\right)+\frac{2+\log(4\dimdom\dimcore)(\dimcore+2)}{2D} < \ndom^2\left(1-\frac{\dimdom}{\ndom}\right)+2\dimdom(\dimcore+2)\log(2\dimdom)\\
\implies&-\frac{\ndom}{\dimdom}\left(1-\frac{\dimdom}{\ndom}\right)-\frac{2+\log(4\dimdom\dimcore)(\dimcore+2)}{2D} > -\ndom^2\left(1-\frac{\dimdom}{\ndom}\right)-2\dimdom(\dimcore+2)\log(2\dimdom)\\
\implies&1-\frac{\ndom}{\dimdom}-\frac{2+\log(4\dimdom\dimcore)(\dimcore+2)}{2D}> -\left(\ndom-\frac{\dimdom}{2}\right)^2 + \frac{\dimdom^2}{4} - 2\dimdom(\dimcore+2)\log(2\dimdom)
\end{align}
    
\end{proof}

\subsection{Proof for \refthm{ood-inv}}
\oodinv*
\begin{proof}
\begin{align}
  \Rood(\thetahatinv) - \Rood(\thetabest)
  &= \sigmaeps^2 + \thetahatinvT\Sigma\thetahatinv + (\thetastar-\thetahatinv)^\top T\left(\thetastar-\thetahatinv\right) - \Rood(\thetabest)\\
  &= \sigmaeps^2 + \thetastarT T \thetastar - \Rood(\thetabest)\\
  &= \sigmaeps^2 + \tau^2\|\thetastar\|^2 - \sigmaeps^2 - \frac{\tau^2}{1+\varratio }\|\thetastar\|^2\\
  &= \frac{\tau^2\varratio \|\thetastar\|^2}{1+\varratio }\\
  &= \frac{\tau^2\varratio \|\thetastarcore\|^2}{1+\varratio }.
\end{align}
\end{proof}

\twocolumn

\section{Extended simulation results}\label{sec:app:simulations}
In this section, we provide additional details about the simulations in \refsec{simulation}, as well as plots of the ID RMSE for both high and low-sample regimes.

\subsection{Additional simulation details}
For all experiments below, we fix $\sigmacore^2 = 0.1, \taucore^2 = 1, \dimcore = 5, \dimspu = 500,$ and $\dimnoise=500$. 
Models are evaluated by their RMSE on two test sets: an ID test set of held-out examples from $\Domtrain$, and an OOD test set that generates examples from \num{1000} new domains $\Domtest$. 
We train with $\ell_2$ regularization; penalty strengths are tuned on an ID validation set.

When applying an augmentation to a training set, we run the augmentation over all inputs \num{5} times, such that the final training set contains $5N$ samples.

We plot ID RMSEs for varying ranges of $\ndom$ in \reffig{simulations_id}.
Training with targeted augmentation results in similar ID error as \shelf and unaugmented training, although targeted augmentations result in slightly higher ID error when $D$ is small. This is because memorizing $\xspu$ improves ID performance. 
Domain-invariant augmentation results in high, constant ID error. Plots are averaged over 10 random seeds with standard errors.

\begin{figure}[h]
    \centering
    \includegraphics[width=\linewidth]{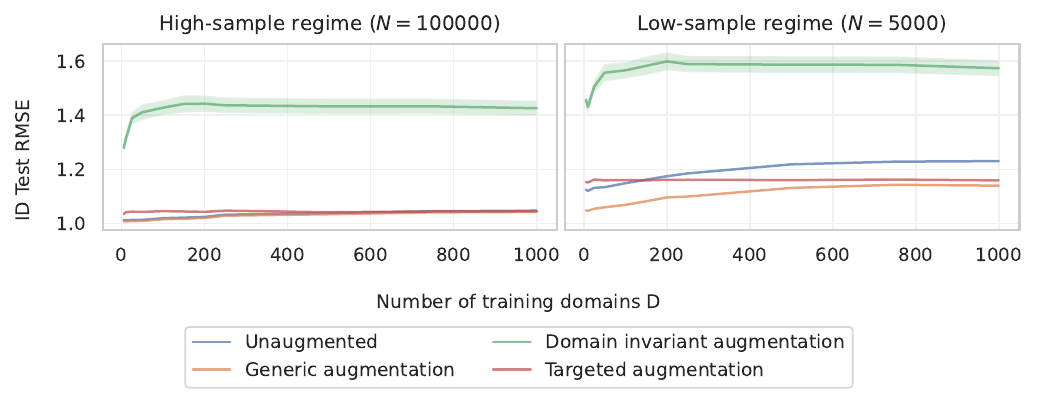}
    \caption{\Id RMSE across values for $\ndom$. Plots are averaged over 10 random seeds with standard errors.}
    \label{fig:simulations_id}
\end{figure}

\section{Experimental details}\label{sec:app:experiment}
In this appendix, we provide tabular forms of results visualized in \reffig{scatterplots}.
We also summarize core experimental details for each dataset, including hyperparameter tuning and model selection protocol.

\subsection{Extended results}

\begin{table}[H]
\caption{Results on \iwildcam}
\label{tab:iwildcam_results}
\resizebox{\linewidth}{!}{
\begin{tabular}{lll}
\toprule
 & ID Test Macro F1 & OOD Test Macro F1 \\
\midrule
Unaugmented & 46.5 (0.4) & 30.2 (0.3) \\
RandAugment & 48.9 (0.2) & 33.3 (0.2) \\
MixUp & 45.5 (0.6) & 28.9 (0.3) \\
CutMix & 45.2 (0.7) & 28.4 (0.5) \\
Cutout & 47.9 (0.7) & 32.6 (0.4) \\
LISA & 45.4 (0.7) & 29.6 (0.4) \\
CDAN & 41.2 (0.6) & 28.6 (0.2) \\
DeepCORAL & 42.4 (1.2) & 30.3 (0.6) \\
IRM & 39.4 (0.4) & 27.8 (0.1) \\
Copy-Paste (Same Y) & \textbf{50.2 (0.7)} & \textbf{36.5 (0.4)} \\
\bottomrule
\end{tabular}
}
\end{table}

\begin{table}[H]
\caption{Results on \camelyon}
\label{tab:camelyon_results}
\resizebox{\linewidth}{!}{
\begin{tabular}{lll}
\toprule
 & ID Val Avg Acc & OOD Test Avg Acc \\
\midrule
Unaugmented & 89.3 (2.0) & 65.2 (2.6) \\
RandAugment & 94.9 (1.0) & 75.3 (1.7) \\
MixUp & 86.9 (2.2) & 69.4 (2.1) \\
CutMix & 84.7 (2.6) & 60.9 (2.2) \\
LISA & 91.0 (1.6) & 73.6 (1.4) \\
DANN & 86.1 (2.1) & 64.5 (1.9) \\
DeepCORAL & 92.3 (1.1) & 62.3 (3.0) \\
IRM & 88.0 (2.3) & 62.4 (3.1) \\
Stain Color Jitter & \textbf{96.7 (0.1)} & \textbf{90.5 (0.9)} \\
\bottomrule
\end{tabular}
}
\end{table}

\begin{table}[H]
\caption{Results on \birdcalls}
\label{tab:birds_results}
\resizebox{\linewidth}{!}{
\begin{tabular}{lll}
\toprule
 & ID Test Macro F1 & OOD Test Macro F1 \\
\midrule
Unaugmented & 70.0 (0.5) & 27.8 (1.2) \\
SpecAugment & 71.4 (0.4) & 22.8 (1.0) \\
MixUp & 74.0 (0.4) & 26.3 (1.0) \\
LISA & 69.7 (0.5) & 29.4 (1.1) \\
Noise Reduction & 75.4 (0.3) & 31.6 (0.9) \\
Random Pass & 71.2 (2.0) & 31.8 (1.2) \\
CDAN & 64.7 (0.5) & 27.0 (1.2) \\
DeepCORAL & 69.2 (0.5) & 27.7 (0.9) \\
IRM & 69.2 (0.4) & 28.3 (0.8) \\
Color Jitter & 73.8 (0.2) & 26.1 (0.9) \\
Copy-Paste + Jitter (Region) & \textbf{75.6 (0.3)} & \textbf{37.8 (1.0)} \\
\bottomrule
\end{tabular}
}
\end{table}

\subsection{Hyperparameters}\label{sec:app:hyperparams}

\tightparagraph{\iwc.}
All experiments used a ResNet-50, pretrained on ImageNet, with no weight decay and batch size 24, following \citet{sagawa2021extending,koh2021wilds}.
Model selection and early stopping was done on the OOD validation split of \iwc, which measures performance on a held-out set of cameras $\Domval$, which is disjoint from both $\Domtrain$ and $\Domtest$.
We tuned all methods by fixing a budget of 10 tuning runs per method with one replicate each; the hyperparameter grids are given in \reftab{iwc_hparam}.
Final results are reported over 5 random seeds. 

For CDAN, we tuned the classifier and discriminator learning rates and fixed the featurizer learning rate to be a tenth of the classifier's, following \citet{sagawa2021extending}. 

We applied all data augmentations stochastically with a tuned \textit{transform probability}, since we found that doing so improved performance as in prior work \citep{gontijo2020affinity}.
For all augmentations, we also stochastically apply a random horizontal flip with the learned transform probability.

\begin{table}[h]
    \caption{Hyperparameter search spaces for methods on \iwildcam.    \vspace{0.5em}\label{tab:iwc_hparam}}
    \centering
    \resizebox{\columnwidth}{!}{
    \begin{tabular}{ll}
    \hline
    \textbf{Method} & \textbf{Hyperparameters} \\ \hline
    ERM & Learning rate $\sim 10^{\text{Uni}(-5, -2)}$ \\ \hline
    Copy-Paste & \begin{tabular}[c]{@{}l@{}}Learning rate $\sim 10^{\text{Uni}(-5, -2)}$\\ Transform probability $\sim \text{Uni}(0.5, 0.9)$\end{tabular} \\ \hline
    LISA & \begin{tabular}[c]{@{}l@{}}Learning rate $\sim 10^{\text{Uni}(-5, -2)}$\\ Transform probability $\sim \text{Uni}(0.5, 0.9)$\\ Interpolation method $\in$ \{MixUp, CutMix\} \end{tabular} \\ \hline
    Vanilla MixUp & \begin{tabular}[c]{@{}l@{}}Learning rate $\sim 10^{\text{Uni}(-5, -2)}$\\ Transform probability $\sim \text{Uni}(0.5, 0.9)$\\ $\alpha \in \{0.2, 0.4\}$\end{tabular} \\ \hline
    Vanilla CutMix & \begin{tabular}[c]{@{}l@{}}Learning rate $\sim 10^{\text{Uni}(-5, -2)}$\\ Transform probability $\sim \text{Uni}(0.5, 0.9)$\\ $\alpha \in \{0.5, 1.0\}$\end{tabular} \\ \hline
    RandAugment & \begin{tabular}[c]{@{}l@{}}Learning rate $\sim 10^{\text{Uni}(-5, -2)}$\\ Transform probability $\sim \text{Uni}(0.5, 0.9)$\\ $k \in \{1, 2\}$\end{tabular} \\ \hline
    Cutout & \begin{tabular}[c]{@{}l@{}}Learning rate $\sim 10^{\text{Uni}(-5, -2)}$\\ Transform probability $\sim \text{Uni}(0.5, 0.9)$\\ Version $\in $ \{Original, Bounding box-aware\}\end{tabular} \\ \hline
    CDAN & \begin{tabular}[c]{@{}l@{}}Classifier learning rate $\sim 10^{\text{Uni}(-5.5, -4)}$\\ Discriminator learning rate $\sim 10^{\text{Uni}(-5.5, -4)}$\\ $\lambda \sim 10^{\text{Uni}(-0.3, 1)}$\end{tabular} \\ \hline
    \end{tabular}
    }
\end{table}

\tightparagraph{\cam.}
All experiments used a randomly initialized DenseNet-121, with weight decay $0.01$ and batch size 168, following \citet{sagawa2021extending,koh2021wilds}.
We also fixed the learning rate to that of \citet{sagawa2021extending}, which was selected by the authors of that paper after a random search over the distribution $10^{\text{Uni}(-4, -2)}$. 
For \cam, we found that the choice of learning rate affected the relative ID vs.~OOD accuracies of methods.
To remove this confounder, we therefore standardized the learning rate across augmentations / algorithms for fair comparison. 
Separately tuning the learning rate for each algorithm did not significantly improve performance. 

Because \cam is class-balanced, we ran experiments on DANN (rather than CDAN). For DANN, we used the learning rate fixed across all methods for the featurizer and set the classifier learning rate to be 10$\times$ higher, following \citet{sagawa2021extending}. 

Because \cam has no ID test split, we report \id performance using the ID Val split. 

Model selection and early stopping was done on the OOD validation split of \cam, which measures performance on a held-out hospital $\Domval$, which is disjoint from both $\Domtrain$ and $\Domtest$.
We tuned remaining hyperparameters by fixing a budget of 10 tuning runs per method with one replicate each; the hyperparameter grids are given in \reftab{camelyon_hparam}.
Because of the large variance in performance between random seeds for some algorithms on \cam \citep{koh2021wilds,miller2021accuracy}, we ran 20 replicates in the final results.

\begin{table}[h]
    \caption{Hyperparameter search spaces for methods on \camelyon.    \vspace{0.5em}\label{tab:camelyon_hparam}}
    \centering
    \resizebox{\columnwidth}{!}{
    \begin{tabular}{ll}
    \hline
    \textbf{Method} & \textbf{Hyperparameters} \\ \hline
    \stainj & \begin{tabular}[c]{@{}l@{}}Augmentation strength $\in$ [0.05, 0.1] \end{tabular} \\ \hline
    LISA & \begin{tabular}[c]{@{}l@{}}Interpolation method $\in$ \{MixUp, CutMix\} \end{tabular} \\ \hline
    Vanilla MixUp & \begin{tabular}[c]{@{}l@{}}$\alpha \in \{0.2, 0.4\}$\end{tabular} \\ \hline
    Vanilla CutMix & \begin{tabular}[c]{@{}l@{}}$\alpha \in \{0.5, 1.0\}$\end{tabular} \\ \hline
    RandAugment & \begin{tabular}[c]{@{}l@{}}$k \in \{1,2\}$\end{tabular} \\ \hline
    Cutout & - \\ \hline
    DANN & \begin{tabular}[c]{@{}l@{}}Discriminator learning rate $\sim 10^{\text{Uni}(-4, -2)}$\\ $\lambda \sim 10^{\text{Uni}(-1, 0)}$ \end{tabular} \\ \hline
    \end{tabular}
    }
\end{table}

\tightparagraph{\birds.}
All experiments used an EfficientNet-B0, pretrained on ImageNet, with batch size 64. 
Model selection and early stopping was done on an ID validation split, which measures performance on a held-out examples from $\Domtrain$.
We tuned all methods by fixing a budget of 10 tuning runs per method with five replicates each; the hyperparameter grids are given in \reftab{birds_hparam}.
Because of its small size, \birds has relatively high variance between results; we thus report final results averaged over 20 random seeds. 

For CDAN, we tuned the classifier and discriminator learning rates and fixed the featurizer learning rate to be a tenth of the classifier's, matching our policy on \iwc.
For all augmentations, we also stochastically apply a random horizontal flip with the learned transform probability.

\begin{table}[h]
    \caption{Hyperparameter search spaces for methods on \birdcalls.\vspace{0.5em}\label{tab:birds_hparam}}
    \centering
    \resizebox{\columnwidth}{!}{
    \begin{tabular}{ll}
    \hline
    \textbf{Method} & \textbf{Hyperparameters} \\ \hline
    ERM & \begin{tabular}[c]{@{}l@{}}Learning rate $\sim 10^{\text{Uni}(-4, -3)}$\\ Weight decay $\in \{0, 0.001, 0.1, 1\}$\end{tabular} \\ \hline
    Copy-Paste & \begin{tabular}[c]{@{}l@{}}Learning rate $\sim 10^{\text{Uni}(-4, -3)}$\\ Weight decay $\in \{0, 0.001, 0.1, 1\}$\\ Transform probability $\sim \text{Uni}(0.5, 0.9)$\end{tabular} \\ \hline
    LISA & \begin{tabular}[c]{@{}l@{}}Learning rate $\sim 10^{\text{Uni}(-4, -3)}$\\ Weight decay $\in \{0, 0.001, 0.1, 1\}$\\ Transform probability $\sim \text{Uni}(0.5, 0.9)$\end{tabular} \\ \hline
    Vanilla MixUp & \begin{tabular}[c]{@{}l@{}}Learning rate $\sim 10^{\text{Uni}(-4, -3)}$\\ Weight decay $\in \{0, 0.001, 0.1, 1\}$\\ Transform probability $\sim \text{Uni}(0.5, 0.9)$\\ $\alpha \in \{0.2, 0.4\}$\end{tabular} \\ \hline
    SpecAugment & \begin{tabular}[c]{@{}l@{}}Learning rate $\sim 10^{\text{Uni}(-4, -3)}$\\ Weight decay $\in \{0, 0.001, 0.1, 1\}$\\ Transform probability $\sim \text{Uni}(0.5, 0.9)$\\ $k \in \{1, 2\}$\\$F \in \{10, 20, \cdots, 100\}$\\$T \in \{10, 20, \cdots, 100\}$\end{tabular} \\ \hline
    Random Pass & \begin{tabular}[c]{@{}l@{}}Learning rate $\sim 10^{\text{Uni}(-4, -3)}$\\ Weight decay $\in \{0, 0.001, 0.1, 1\}$\end{tabular} \\ \hline
    Noise Reduction & \begin{tabular}[c]{@{}l@{}}Learning rate $\sim 10^{\text{Uni}(-4, -3)}$\\ Weight decay $\in \{0, 0.001, 0.1, 1\}$\end{tabular} \\ \hline
    CDAN & \begin{tabular}[c]{@{}l@{}}Classifier learning rate $\sim 10^{\text{Uni}(-5, -2)}$\\ Weight decay $\in \{0, 0.001, 0.1, 1\}$\\ Discriminator learning rate $\sim 10^{\text{Uni}(-5, -2)}$\\ $\lambda \sim 10^{\text{Uni}(-0.3, 1)}$\end{tabular} \\ \hline
    \end{tabular}
    }
\end{table}

\subsection{CLIP Experiments}
In our experiments finetuning CLIP on \iwc and \cam, we used OpenAI's CLIP ViT-L/14 at 224 x 224 pixel resolution. Early stopping and model selection were done on the OOD validation splits.  Hyperparameters are given in \reftab{clip_hparam_iwc} for \iwc and \reftab{clip_hparam_cam} for \cam; we based \cam hyperparameters on \citet{kumar2022fine} and \iwc hyperparameters on \citet{wortsman2022model}. We tuned all methods by fixing a budget of 10 tuning runs per method. Results are averaged over five seeds.

\begin{table}[h]
    \caption{Hyperparameter search spaces for CLIP experiments on \iwc.\vspace{0.5em}\label{tab:clip_hparam_iwc}}
    \centering
    \resizebox{\columnwidth}{!}{
    \begin{tabular}{ll}
    \hline
    \textbf{Method} & \textbf{Hyperparameters} \\ \hline
    ERM & \begin{tabular}[c]{@{}l@{}}Learning rate $\sim 10^{\text{Uni}(-6, -4)}$\\ Weight decay $\sim 10^{\text{Uni}(-4, -0.2)}$\\ Optimizer = AdamW\end{tabular} \\ \hline
    \cpsamey & \begin{tabular}[c]{@{}l@{}}Learning rate $\sim 10^{\text{Uni}(-6, -4)}$\\ Weight decay $\sim 10^{\text{Uni}(-4, -0.2)}$\\Transform probability  $\sim \text{Uni}(0.5, 0.9)$\\ Optimizer = AdamW\end{tabular} \\ \hline
    \end{tabular}
    }
\end{table}

\begin{table}[h]
    \caption{Hyperparameter search spaces for CLIP experiments on \cam.\vspace{0.5em}\label{tab:clip_hparam_cam}}
    \centering
    \resizebox{\columnwidth}{!}{
    \begin{tabular}{ll}
    \hline
    \textbf{Method} & \textbf{Hyperparameters} \\ \hline
    ERM & \begin{tabular}[c]{@{}l@{}}Learning rate $\sim 10^{\text{Uni}(-6, -3)}$\\ Weight decay $= 0.01$\\ Optimizer = SGD\end{tabular} \\ \hline
    \stainj & \begin{tabular}[c]{@{}l@{}}Learning rate $\sim 10^{\text{Uni}(-6, -3)}$\\ Weight decay $= 0.01$\\Augmentation strength $\in$ [0.05, 0.1]\\ Optimizer = SGD\end{tabular} \\ \hline
    \end{tabular}
    }
\end{table}

\section{Additional related work}\label{sec:app:related_work}

\tightparagraph{Data augmentations for OOD robustness.}
Prior work has sought to design augmentations specifically for robustness \citep{puli2022nuisances,wang2022out}.
Many augmentations are inspired by domain invariance and aim to randomize all \domaindependent features, including robust features $\xcore$.
For example, inter-domain MixUp interpolates inputs from different domains, possibly within the class~\cite{wang2020heterogeneous,xu2020adversarial,yan2020improve,yao2022improving}. \citet{ilse2021selecting} propose to select transformations which maximally confuse a domain classifier.
Several works train generative models to transform images between domains by learning to modify all \domaindependent features~\cite{hoffman2018cycada,zhou2020learning,robey2021model}.
In contrast, we preserve $\xcore$ in targeted augmentations.

\tightparagraph{Targeted augmentations in the applied literature.}
Many existing domain-specific augmentations can fit the proposed framework of targeted augmentations.
For example, \stainj is sourced from the biomedical literature and was designed for OOD robustness \citep{tellez2018whole,tellez2019quantifying,miller2021accuracy}.
\cp (non-selective) has been previously applied to a smaller, single-habitat camera trap dataset~\citep{beery2020synthetic}.
Our contribution lies in interpreting and formalizing why these targeted augmentations are effective OOD.

\tightparagraph{Underspecification.}
\citet{d2020underspecification} point out the underspecification issue in \ood generalization, in which multiple models are optimal on the training data, but generalize very differently \oodnodash.
While our theoretical setting does not precisely fit the above definition of underspecification, we observe a related phenomenon;
although there is a unique optimal model due to feature noise, OOD error can be high when the noiseless version of the regression problem is underspecified.

\end{document}